\documentclass{article}

\usepackage[final]{neurips_2023}

\usepackage[utf8]{inputenc} % allow utf-8 input
\usepackage[T1]{fontenc}    % use 8-bit T1 fonts
\usepackage[dvipsnames]{xcolor}         % colors
\usepackage[colorlinks]{hyperref}       % hyperlinks
\usepackage{url}            % simple URL typesetting
\usepackage{booktabs}       % professional-quality tables
\usepackage{amsfonts}       % blackboard math symbols
\usepackage{nicefrac}       % compact symbols for 1/2, etc.
\usepackage{microtype}      % microtypography
\usepackage{graphicx}
\usepackage{wrapfig}
\usepackage{esvect}

\hypersetup{
linkcolor=BrickRed
,citecolor=MidnightBlue
,filecolor=RedViolet
,urlcolor=RedViolet
,menucolor=BrickRed
,runcolor=RedViolet
,linkbordercolor=BrickRed
,citebordercolor=MidnightBlue
,filebordercolor=RedViolet
,urlbordercolor=RedViolet
,menubordercolor=BrickRed
,runbordercolor=RedViolet
}

\everypar{\looseness=-1}
\usepackage{amsmath,amsfonts,bm}

\def\eqref#1{equation~\ref{#1}}
\def\1{\bm{1}}

\DeclareMathAlphabet{\mathsfit}{\encodingdefault}{\sfdefault}{m}{sl}
\SetMathAlphabet{\mathsfit}{bold}{\encodingdefault}{\sfdefault}{bx}{n}

\def\gL{{\mathcal{L}}}
\def\gM{{\mathcal{M}}}
\def\gN{{\mathcal{N}}}

\def\gU{{\mathcal{U}}}

\def\gX{{\mathcal{X}}}
\def\gY{{\mathcal{Y}}}
\def\gZ{{\mathcal{Z}}}

\def\sR{{\mathbb{R}}}

\newcommand{\E}{\mathbb{E}}

\newcommand{\KL}{D_{\mathrm{KL}}}

\newtheorem{proposition}{Proposition}[section]

\newtheorem{lemma}{Lemma}[section]

\newenvironment{proof}{\paragraph{Proof:}}{\hfill$\square$}

\DeclareMathOperator*{\argmax}{arg\,max}

\title{Topological Obstructions and How to Avoid Them}

\author{%
  Babak Esmaeili\thanks{Equal contribution} \\
  Generative AI Group\\
  Eindhoven University of Technology \\
  \texttt{b.esmaeili@tue.nl} \\
  \And
  Robin Walters$^*$ \\
  Khoury College of Computer Sciences \\
  Northeastern University \\
  \texttt{r.walters@northeastern.edu} \\
  \AND
  Heiko Zimmermann \\
  Amsterdam Machine Learning Lab \\
  University of Amsterdam \\
  \texttt{h.zimmermann@uva.nl} \\
  \And
  Jan-Willem van de Meent \\
  Amsterdam Machine Learning Lab \\
  University of Amsterdam \\
  \texttt{j.w.vandemeent@uva.nl} \\
}

\begin{document}

\maketitle

\begin{abstract}
%Geometric inductive biases such as spatial curvature, factorizability, or equivariance can aid learning of representations that better reflect the latent structure of a dataset, which in turn can improve generalization in downstream tasks. 
Incorporating geometric inductive biases into models can aid interpretability and generalization, but encoding to a specific geometric structure can be challenging due to the imposed topological constraints. In this paper, we theoretically and empirically characterize obstructions to training encoders with geometric latent spaces. We show that local optima can arise due to singularities (e.g.~self-intersection) or due to an incorrect degree or winding number. We then discuss how normalizing flows can potentially circumvent these obstructions by defining multimodal variational distributions. Inspired by this observation, we propose a new flow-based model that maps data points to multimodal distributions over geometric spaces and empirically evaluate our model on 2 domains. We observe improved stability during training and a higher chance of converging to a homeomorphic encoder.
 %which in turn improves the stability of training and convergence to a homeomorphic mapping. We perform empirical evaluations in 2 domains, which demonstrate that flow-based models succeed at circumvent the identified optimization obstructions.
 %which in turn improves the stability of training and convergence to a homeomorphic mapping. We perform empirical evaluations in 2 domains, which demonstrate that flow-based models succeed at circumvent the identified optimization obstructions.
\end{abstract}

\section{Introduction}
\label{sec:intro}

A well-established idea in machine learning research is that geometric inductive biases can help us learn representations that reflect the underlying structure of a dataset \citep{bronstein2021geometric,higgins2022symmetry}.
A key intuition behind this line of research is that such representations make it easier to reason about the similarities of different instances in the dataset, for example by relating inputs using a rotation or translation, which in turn aids interpretability and data-efficiency.
Geometric inductive biases have been explored in a wide variety of forms, including models that are defined on hyperbolic or spherical Riemannian manifolds \citep{lezcano2019cheap, ganea2018hyperbolic}, models that are equivariant or invariant with respect to particular symmetry groups
%\citep{pmlr-v32-cohen14, laptev2016ti}
\citep{kondor2018generalization, cohen2016group}, and work that leverages symmetries to learn disentangled representations that factorize into distinct axes of variation \citep{higgins_towards_2018}.

In this paper, we consider symmetry-based approaches to learning representations in an  unsupervised manner by imposing geometric inductive biases on the representation space. 
% In this context, the notion of a representation that ``reflects the underlying structure of the data'' that we will consider is a representation that is \emph{homeomorphic} to the true generative factors. 
In this context, our notion of a representation that ``reflects the underlying structure of the data'' is a representation that is \emph{homeomorphic} to the true generative factors. 
We are specifically interested in the optimization challenges that arise when encoding to geometric spaces such as Lie groups. 
% In this context, we are particularly interested in recent work on Variational Autoencoders . 
% In this context, we are particularly interested in recent work on homeomorphic autoencoders ~\citep{falorsi_explorations_2018,falorsi2019reparameterizing}. 
While a common intuition is that an inductive bias that matches the underlying topology of the data will guide a model towards a homeomorphic representation, there are also indications that certain inductive biases can make a model more difficult to train in practice~\citep{park2022learning,falorsi_explorations_2018,batson2021topological}, which limits their practical utility. 
% This potentially limits their practical utility when our goal is to uncover underlying structure from data.  

%Counter to the standard intuition, we find that even when the inductive bias is ``correct'' in the sense that it matches the topology of the data, it can often be the case that it interferes with optimization.  That is, the imposed constraints correctly identify good local minimums, however, they make getting from random initializations to these minima more difficult.  Although this issue has been alluded to before ~\citep{de_haan_topological_2018}, we suspect it is understudied due to a publication bias towards positive results.  That is, many authors attempt to impose inductive biases which sometimes surprisingly fail to work and thus move on to other methods without examining the cause of failure.  

To understand why encoding to geometric structures can give rise to optimization challenges, we formalize topological defects that can occur in a randomly initialized encoder, such as discrepancies in the winding number or crossing number relative to those in a homeomorphic encoder. 
We show that these topological defects will be preserved under continuous optimization, which suggests that escaping these local optima relies on the discrete jumps that are employed during optimization. 
% Moreover, we show that, even in the absence of topological defects, a homeomorphic encoder may violate isometry, which is to say that equidistant points on the abstract manifold may not map to equidistant points in the latent space.

\begin{figure}[!t]
\begin{center}
\includegraphics[width=0.95\textwidth]{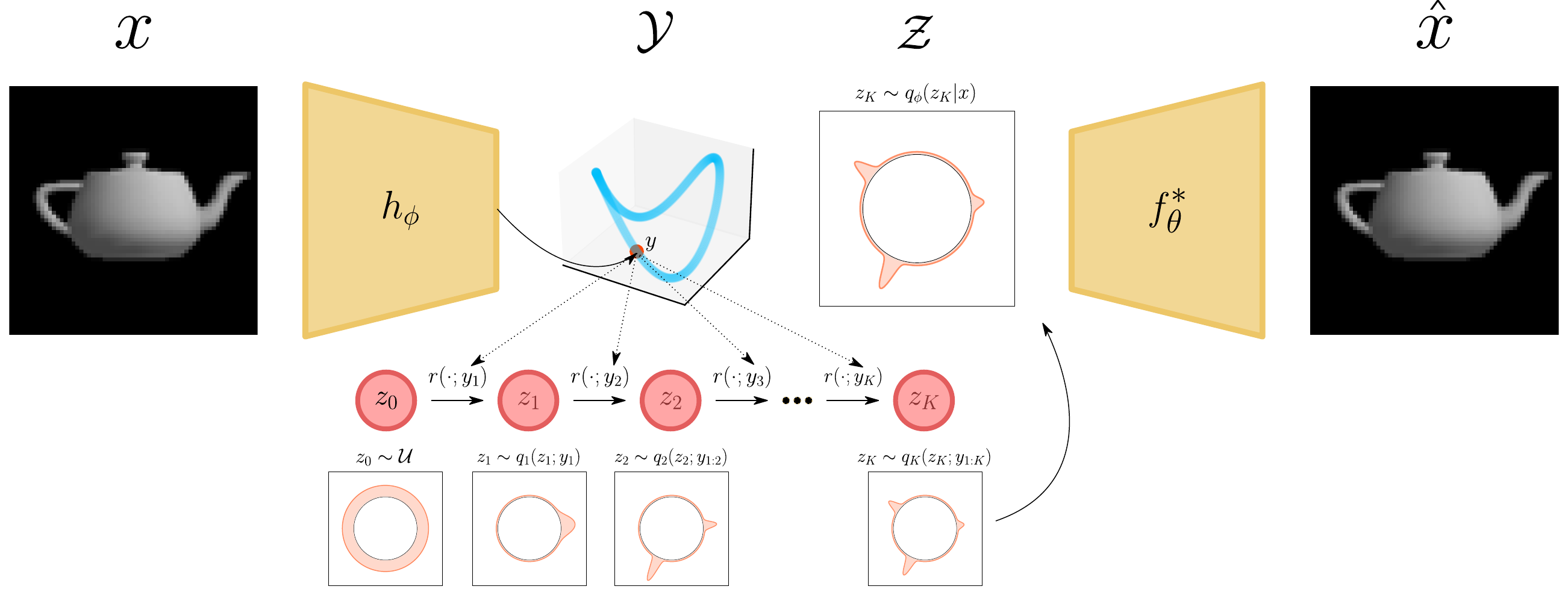}
\end{center}
\vspace{-1ex}
\caption{A GF-VAE consists of an encoder network $h_\phi$ that maps data $x$ to an intermediate parameter space $\gY$. The encoded vector $y$ is split into $K$ parts, where each sub-vector $y_k$ corresponds to the parameters of a normalizing flow at layer $k$. We can sample from the variational distribution $q_\phi$ by first sampling from a uniform prior defined on the Lie group ($ \mathrm{SO}(2)$ in this case), followed by a sequence of $K$ bijective transformations $r(\cdot; y_k)$. The output of the normalizing flows, denoted as $z_K$, is then passed to a decoder $f^{*}_{\theta}$ that maps from the Lie group $\gZ$ to the data space $\gX$.}
%\vspace*{-0.5ex}
\label{fig:overview}
\end{figure}

To circumvent these obstacles to optimization, we propose Group-Flow Variational Autoencoders (GF-VAEs), which leverage normalizing flows to model complex multimodal distributions on Riemannian manifolds (Figure~\ref{fig:overview}). We show that if we define the mode of the variational distribution to be the representation, normalizing flows can circumvent some of the challenges associated with local optima due to their multimodal nature. Experiments demonstrate that GF-VAEs can escape local optima during the early stages of training, resulting in more reliable convergence to a homeomorphic mapping and a greater degree of continuity after training. 
% We additionally propose an isometric objective that encourages equidistant points on the abstract manifold to map to equidistant points in latent space.

% GF-VAEs leverage Normalizing Flows, a technique commonly used in VAEs to mitigate amortization gaps~\citep{rezende2015variational}, to model complex distributions on Riemannian manifold

% GF-VAEs leverage Normalizing Flows, a technique commonly used in VAEs to mitigate amortization gaps~\citep{rezende2015variational}, to model complex distributions on Riemannian manifolds

% To alleviate these challenges to optimization, we propose an autoencoder objective (Section~\ref{sec:isom-ae}) that encourages equidistant points on the abstract manifold to map to equidistant points in latent space. Experiments demonstrate that this objective can help to escape local optima during the early stages of training, resulting in more reliable convergence and a greater degree of isometry after training.

We summarize the main contributions of this paper as follows: 
\begin{itemize}
    \item We characterize topological defects that can arise in encoders that map onto spaces with geometric structure. We show that some obstructions that arise from topological defects cannot be resolved using continuous optimization.
    % \item We define evaluation criteria based on the winding number, crossing number and a novel notion of \emph{approximate isometry} to measure topological defects and isometry violations in the encoder.
    \item We define evaluation criteria based on the winding number, crossing number, and continuity to measure topological defects and homeomorphism violations in the encoder.
    \item We propose GF-VAEs, a new VAE-based model that that employs normalizing flows to define complex distribution on Lie groups. We  empirically show that GF-VAEs are able to aid in circumventing identified optimization obstructions.
    % \item We introduce an isometric autoencoder objective that maps sequences of equidistant inputs to sequences of equidistant points in the latent space, and empirically demonstrate improved  convergence and isometry of the resulting representation.
\end{itemize}

% \vspace{-0.5\baselineskip}
% %%%%%%%%%%%%%%%%%%%%%%%%%%%%%%
% \section{Background}
% \label{sec:background}
% %%%%%%%%%%%%%%%%%%%%%%%%%%%%%%

%In this section, we will describe the preliminary topics that are used throughout the paper. 

% \paragraph{Variational Autoencoders.} A variational autoencoder (VAE)~\citep{kingma2014,rezende2014stochastic} combines a generative model $p_\theta(x,z)$ and an inference model $q_\phi(z|x)$ defined on data space $\gX$ and latent space $\gZ$. The generative model comprises a prior $p(z)$ and a likelihood model $p_\theta(x|z)$ that is typically parameterized by a neural network. The inference model is also parameterized by a neural network that maps data points to a variational distribution $q_\phi(z|x)$. The two models are jointly optimized by maximizing the evidence lower-bound a.k.a ELBO,
% \begin{equation}
% \label{eq:elbo}
%     \gL_{\phi,\theta}^{\text{VAE}}(x) 
%     =
%     \E_{q_\phi(x|z)}
%     \left[
%         \log p_{\theta}(x|z)
%     \right]
%     -
%     \KL\left[q_{\phi}(z|x) \| p(z)\right] 
%     \leq
%     \log p_\theta(x).
% \end{equation}

%
%The Lie algebra structure provides a way to associate each element of the algebra with a vector field on the group manifold via an \emph{exponential map}. The exponential map maps an element of the algebra to a group element that lies at a distance of one from the identity, following the flow of the associated vector field. 
%

%\vspace{-0.5\baselineskip}
%%%%%%%%%%%%%%%%%%%%%%%%%%%%%%
\section{Problem Statement}
\label{sec:problem}
%%%%%%%%%%%%%%%%%%%%%%%%%%%%%%

% \begin{wrapfigure}{r}{0.5\textwidth}
% % \centering
% \includegraphics[width=0.5\textwidth]{Figures/overview.png}
% \caption{A homeomorphic encoding from the data manifold in $\sR^n$ (\emph{Left}) to $SO(2)$ (\emph{Right}).}
% \label{fig:overview}
% \end{wrapfigure}

% \vspace{-0.5\baselineskip}

\paragraph{Homeomorphic Encoders.} Our goal is to learn representations in domains where we have prior knowledge of the geometric structure, specifically structure in the form of a symmetry group that can be associated with the input data.
%We focus on learning representations in the case where we have good a priori knowledge of the geometric structure of the input data, but not concrete knowledge of 
We assume that data lies on a low-dimensional manifold $\gM$ that is embedded in a higher-dimensional space $\gX := \mathbb{R}^n$ via a mapping $g_{\gX}: \gM \to \gX$. This is commonly known as the \emph{manifold hypothesis}~\citep{bengio2013representation}.
We denote the image of the mapping by $\gM_x := g_{\gX}(\gM) \subseteq \gX$. Then $g_{\gX}$ is a  homeomorphism, or topological isomorphism, onto its image ~$g \colon \gM \stackrel{\sim}{\longrightarrow} \gM_x$.  That is, $g$ is continuous, bijective, and has a continuous inverse. 

% We assume that the latent space $\gZ$ can be associated with a Lie group.
We wish to learn an encoder $f_\phi \colon \gX \to \gZ$ such that its restriction $f_\phi|_{\gM_x} \colon \gM_x \to \gZ$ is a homeomorphism. Following \citet{falorsi_explorations_2018}, we will define this mapping in terms of a network $h_\phi \colon \gX \to \gY$ that maps to an intermediate space $\gY := \mathbb{R}^d$, followed by a known projection $\pi: \gY \to \gZ$,
\begin{equation}\label{eqn:fphi}
    f_{\phi} \colon \gX  \xrightarrow[]{h_{\phi}} \gY \xrightarrow[]{\pi} \gZ.
\end{equation}

\paragraph{Lie Groups.} 
%The concept of symmetry-based transformations i.e., transformations that alter some aspects while leaving others unchanged, are described in mathematics by \emph{groups}. A group $G$ is a set equipped with a binary operation $\cdot: G \times G \rightarrow G$ which satisfies four main axioms: closure, associativity, identity, and invertibility.
Our work focuses the specific case where the manifolds $\gM$ and $\gZ$ are Lie groups.
% , that is differentiable manifolds encoding continuous symmetries.
Any set of symmetries may be formally described by a \emph{group} $G$, which is a set of invertible transformations which may be composed using a binary operation $\cdot: G \times G \rightarrow G$. A \emph{Lie group} is a group that is also a differentiable manifold. Lie groups describe continuous symmetries such as rotations and translations, and are therefore a natural mathematical setting for any system with spatio-temporal symmetries.

Reasoning about Lie groups is difficult due to  their non-flat structure. \emph{Lie algebras} provide a way to study Lie groups by considering the tangent space $\mathfrak{g}$ of the manifold at the identity element. The exponential map $\exp : \mathfrak{g} \rightarrow G$ maps points the Lie algebra to points on the group manifold as such.  A vector $v \in \mathfrak{g}$ defines a vector field on $G$ using the group action to transport $v$ around $G$.  Then $\exp(v)$ is defined to be the point reached by flowing along this vector field for unit time. For more details on Lie groups, we refer the readers to~\citep{hall2013lie}.

\paragraph{Variational Autoencoders on Lie Groups.} 
% We will primarily consider the setting in which we train the encoder in an unsupervised manner using a variational autoencoder (VAE)~\citep{kingma2014,rezende2014stochastic}.
In this paper, we primarily focus on variational autoencoder (VAE)~\citep{kingma2014,rezende2014stochastic} as a means to learn a homeomorphic embedding.
In this setting, we define a generative model by composing a uniform prior $p(z) = \gU(z)$ on $\gZ$ with a likelihood model $p_\theta(x|z) = \gN(x; f^{*}_\theta(z) ,\sigma_x^2 I_n)$ that is defined in terms of a decoder network $f^{*}_\theta : \gZ \to \gX$. We use the encoder $f_\phi$ to define a variational distribution $q_\phi(z|x)$, whose design we discuss below and train the encoder and decoder jointly by optimizing the variational lower bound,
\begin{equation}
\label{eq:elbo}
    \gL_{\phi,\theta}^{\text{VAE}}(x) 
    =
    \E_{q_\phi(x|z)}
    \left[
        \log p_{\theta}(x|z)
    \right]
    -
    \KL\left[q_{\phi}(z|x) \| p(z)\right] 
    \leq
    \log p_\theta(x).
\end{equation}
Defining a variational distribution on a manifold is generally not straightforward as it requires finding an expression of the density on the manifold or keeping track of the change of volume when projecting to the manifold from the tangent space. \citet{falorsi2019reparameterizing} define a reparameterized construction for sampling from a Gaussian distribution on $\gZ$ by first sampling $\epsilon \sim \gN(0, I_p)$ from the $p$-dimensional Lie algebra associated with $\gZ$, then rescaling $\epsilon$ by way of element-wise multiplication $\sigma_\phi(x) \odot \epsilon$ using a network $\sigma_\phi: \gX \to \mathbb{R}^p$, and computing the corresponding element $z_\epsilon$ on the group manifold using the exponential map. By left multiplying this randomly sampled $z_\epsilon$ by the group element $f_\phi(x) \in \gM$ to \emph{move} the mode of the final distribution to its intended location, we obtain the construction
% Defining a variational distribution $q_\phi(z|x)$ on the manifold $\gM$ is a not trivial, since we need to reparameterize this distribution to train the model. We follow the construction by \citet{falorsi2019reparameterizing}, which defines $z = f_\phi(x) \cdot z_\epsilon$ by combining the element of the Lie group $f_\phi(x)$ returned by the homeomorphic encoder with a random element $z_\epsilon$. To sample this random element in a manner that admits reparameterization, we first sample $\epsilon \sim \gN(0, I_p)$ from the $p$-dimensional Lie algebra associated with $\gZ$, rescale $\epsilon$ by way of element-wise multiplication $\sigma_\phi(x) \cdot \epsilon$ using a network $\sigma_\phi: \gX \to \mathbb{R}^p$, and compute the corresponding element $z_\epsilon$ in the Lie group $\gZ$ using the exponential map. This defines the reparameterized construction
\begin{align}
    \label{eq:reparam-sampling-lie-group}    
    \epsilon \sim \gN(\cdot; 0, I_p)
    ,&&
    z_{\epsilon} = \exp (\sigma_\phi(x) \odot \epsilon)
    ,&&
    z = f_\phi(x)  \cdot z_{\epsilon}
    ,&&
    p(z_\epsilon) = p(\epsilon) \left| \det \frac{\partial \exp(\epsilon)}{\partial \epsilon} \right|^{-1}
    .
\end{align}
We account for the possible change in volume using the change of variable formula.
Note that composing group elements $f_\phi(x)  \cdot z_{\epsilon}$ does not change the density of the resulting point on the manifold because the \emph{Haar} measure, a standard choice for Lie groups, is left invariant.
%, in the case of Lie groups, one approach is to define a distribution  for some network $\sigma^{2}_z: \gX \rightarrow \sR^{p}$ on the Lie algebra $\sR^{p}$ which is a vector space instead followed by applying the exponential map. We can account for the possible change in volume using the change of variable formula. We can obtain a sample on the group by first sampling from the distribution defined on Lie algebra and using the exponential map to project it on the Lie group followed by left multiplying by the group element $f_\phi(x) \in \gM$.
% \begin{equation}
% \label{eq:reparam-sampling-lie-group}    
% \epsilon \sim \gN(\cdot; 0, 1), \qquad g_{\epsilon} = \exp (\sigma_\phi(x) \: \epsilon), \qquad z = f_\phi(x)  \cdot g_{\epsilon}.
% \end{equation}
%~\citep{falorsi2019reparameterizing}.  

\paragraph{Running Example.} As a concrete running example, which we will use throughout this paper, we will consider data on the circle $\gM = \mathrm{SO}(2)$ that is embedded into a space of images $\gX$. 
%In Figure~\ref{fig:so2-failed-cases}, we consider the space $\gX := \mathbb{R}^{1024}$ of flattened $32 \times 32$ grayscale images.  
The subset $\gM_x$ of images 
%of rotated Tetrominoes $\gX_\mathrm{rt} \subset \gX$
generated by a function $g \colon \mathrm{SO}(2) \to \gM_x$ corresponds to images of an object that is subject to a one-dimensional rotation. We will consider the case of in-plane rotations of images, as well images of a 3-dimensional object that is rotated around a single axis.
%that renders a 3-dimensional teapot under rotations along a particular axis, as shown in Figure \ref{fig:so2-failed-cases} left. 

The special orthogonal Lie group $\mathrm{SO}(2)$ is defined as  
\begin{equation*}
    \mathrm{SO}(2) := \{ z \ | \ z \in \mathrm{GL}(2),  z^Tz = I,  \det(z) = 1\}
    = \biggl\{
    A(y) :=
    \begin{bmatrix} 
        y_1  &  -y_2 \\
        y_2 & y_1 
    \end{bmatrix} 
    \ | \ 
    y \in \mathbb{R}^2,
    \lVert y \rVert = 1
    \biggr\},
\end{equation*}
where $\mathrm{GL}(2)$ is the general linear group, which is the group of invertible $2 \times 2$ matrices under matrix multiplication. 
% Define, for example, the embedding $g \colon \mathrm{SO}(2) \to \mathbb{R}^2$ by $z \mapsto z (0, 2)^\top$, embedding the manifold as a circle of radius 2.  We consider $f_\phi$ mapping this embedding back to the group by first encoding to $\gY \subseteq \sR^2$ via $h_{\phi}$ and then projecting to the circle $S^1$, the manifold underlying the Lie group, via $\pi \colon \sR^2 \to \mathrm{SO}(2)$ defined $y \mapsto A(y/\|y\|)$. 
The set of images $\gM_x$ lives on a manifold that is homeomorphic to the rotation group $\mathrm{SO}(2)$, embedded into the space of images by $g$.
%We wish to train an encoder  $f_\phi$ that maps from images to points on the circle $S^1$ which is isomorphic to $\mathrm{SO}(2)$. To do so, we follow the design proposed by~\citet{falorsi_explorations_2018},
%\begin{equation}\label{eqn:fphi}
%    f_{\phi} \colon \gX  \xrightarrow[]{h_{\phi}} \gY \xrightarrow[]{\pi} \gZ.
%\end{equation}
Because we assume we know the underlying manifold $\gM$, we can design $\gZ$ to have the same structure. When $\gM = \mathrm{SO}(2)$, we can use $\gY := \sR^2$. The function $h_\phi$ in Equation~\ref{eqn:fphi} is then an ordinary neural network, and the projection to $\gZ$ is simply the projection onto the unit circle $\colon \sR^2 \to \mathrm{SO}(2) := y \mapsto A(y/\|y\|)$.

\begin{figure}[!t]
    \centering
    % \includegraphics[width=\textwidth]{Figures/so2_z_all_cases2.png}
    % \includegraphics[width=\textwidth]{Figures/so2_failure_cases_examples.pdf}
    % \includegraphics[height=1.0in]{Figures/so2_teapot_failure_cases_examples_crop1.pdf}
    % \hspace{2ex}
    % \includegraphics[height=1.0in]{Figures/so2_teapot_failure_cases_examples_crop2.pdf}
    \includegraphics[width=0.99\textwidth]{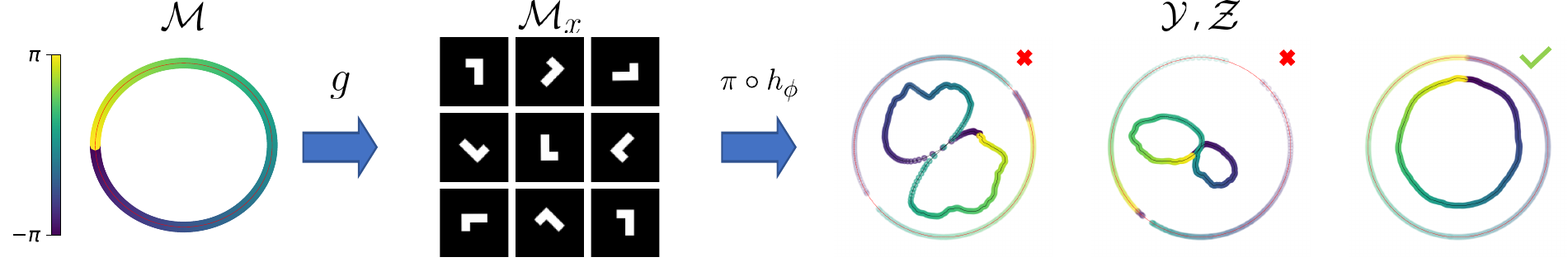}
    %\vspace{-0.5\baselineskip}
    \caption{Example of topological defects in learned encoders for a VAE with data on $\gX = \sR^{32 x 32}$ in the form of rotations L-shaped tetrominoes and $\gZ=\mathrm{SO}(2)$ (the unit circle). Our goal is to learn an encoder $f_\phi$ that defines a homeomorphism (a continuous bijection with continous inverse) between the manifold of images of L-shaped tetrominoes $\gM_x \subset \gX$ and that of the latent space $\gZ=\mathrm{SO}(2)$. The encoder $f_\phi = \pi \circ h_\phi$ is defined by composing a network $h_\phi:\gX \to \gY$ with a projection $\pi:\gY \to \gZ$. On the right, we show the intermediate space $\gY_x = h_\phi(\gM_x)$ and its projection $\gZ_x = \pi(\gY_x)$ for 3 random seeds after convergence. Colours indicate the angle associated with each data point on the manifold $\gM$. Optimization obstructions can arise when the network $h_\phi$ maps data onto a trajectory $\gY_x$ that exhibits topological defects, such as the crossing in a ``figure 8'' shape, which gives rise to discontinuities in the projection $\gZ_x$ onto the latent space.
    }
    %We assume there is a (unknown) homeomorphic mapping $g$ from the ground-truth manifold $\gM$ to the data space. An encoder network $h_\phi$ maps from data space $\gX$ (e.g.~images of rotated L-shaped tetromino) to an intermediate space $\gY$, followed by a projection $\pi$ onto a latent space $\gZ$ with a non-trivial topology (e.g.~the unit circle $S^1$). Optimization obstructions can arise when the encoder maps data onto a trajectory with topological defects, such as the crossing in a ``figure 8'' shape, which manifest as discontinuities after projection. On the right, we show intermediate space $\gY_x = h_\phi(\gM_x)$ after convergence for 3 random seeds.}
    % \caption{{\color{red} TODO: UPDATE FIGURE.} Intermediate space $\gY_x = h_\phi(\gM_x)$ after convergence for 3 seeds.}
    \label{fig:so2-failed-cases}
    %\vspace*{-0.25\baselineskip}
\end{figure}

%%%%%%%%%%%%%%%%%%%%%%%%%%%%%%

\vspace{-0.5\baselineskip}
%%%%%%%%%%%%%%%%%%%%%%%%%%%%%%
\section{Optimization  Obstructions}
\label{sec:theory}

In practice, training homeomorphic encoders can give rise to optimization challenges. To develop intuition for these challenges, we will consider the running example $\gM=\mathrm{SO}(2)$ with data in the form of in-plane rotated images on $\gX = \sR^{32 \times 32}$ (Figure~\ref{fig:so2-failed-cases}). We train a VAE in which the network $h_\phi$ is a standard convolutional network which is paired with a deconvolutional decoder $f^{*}: \mathrm{SO}(2) \to \sR^{32\times 32}$ (see Appendix~\ref{app:sec:exp}). 
% What types of latent spaces result from training this model with the ELBO? 
Figure~\ref{fig:so2-failed-cases} shows encodings in the intermediate space $\gY$ after training with ELBO  with 3 random seeds. For the first two seeds, we see that $h_\phi(\gM_{x})$ crosses over itself, resulting in an figure ``8'' shape. When projected onto $\gZ := \mathrm{SO}(2)$, this results in discontinuities at the crossover points. The second seed in addition also exhibits a sparse region in the intermediate space, leading to a gap in the $\mathrm{SO}(2)$ projection. Only the third initialization has converged to the correct topology. These local optima are not unique to this example; obstructions have also been encountered in \cite{falorsi_explorations_2018} when trying to learn 3D orientations of a rotating multi-color cube from 2D images using a homeomorphic VAE. \citet{park2022learning} show that the homeomorphic VAE cannot generalize well to other shapes and tends to learn a degenerate embedding to a small part of $\mathrm{SO}(3)$. 

The main observation that we make in this paper is that imposing a geometric structure on the latent space can introduce topological obstructions \emph{during optimization}. This insight is distinct from the homological obstructions identified by \citet{de_haan_topological_2018}, who describe the obstructions that emerge from the choice of parameterization on the Lie group. We will refer to the topological defects that we identify in this paper as ``optimization obstructions''.  The problem that we identify here is that randomly-initialized layers have a high probability of exhibiting topological defects (degree, crossing, etc.) that cannot be resolved under continuous optimization using gradient flow. Removal of these topological defects is thus only possible by relying on the jumps coming from performing SGD with a large enough learning rate. This implies that while escaping such local minima is possible, it is difficult and may require many epochs to do so, dramatically slowing training and undercutting the advantages of learning a homeomorphic representation.  

We now discuss several specific optimization obstructions.
%which we associate with local minima and topological obstructions to optimization 
We focus on the case where $\gM$ is the Lie group $\mathrm{SO}(2)$, with the same example setting described in Section~\ref{sec:problem}. All the optimization obstructions we consider occur in this case and in the case of higher-dimensional manifolds $\gM$ as well, but are simpler to describe for $S^1$.  
% Note that the data space $\gX$ can be a high-dimensional space such as image space.
See Appendix \ref{app:sec:proof} for proofs.
%and the embedded manifold $\gM_x$ a set of images with a cyclic factor of variation. 

\subsection{Figure Eight Local Minima}
\label{subsec:fig-eight}

To more precisely describe obstructions that might arise during optimization, we consider continuous-time training along a gradient flow. We denote the weights of the initialized encoder as $\phi(0)$ and the trained weights as $\phi(1)$.  
We consider the idealized setting in which $\phi(t)$ is a continuous function of $t$. 
% That is, $f_{\phi(t)}(x)$ is a homotopy from the initialized encoder $f_{\phi(0)}(x)$ to the trained one $f_{\phi(1)}(x)$.

%Empirically, after some training, we often observe a ``figure 8'' pattern in $\gY$ of roughly the form $(h^\infty \circ g)(\theta) = \sec^4(\theta)\cos(2\theta)$.   In this case, the embedding into $\gZ$ contains two disjoint pieces $[-\pi/4,\pi/4] \cup [3\pi/4,5\pi/4]$ parameterized
Empirically, either at initialization or after some training, we often observe a ``figure 8'' pattern in $\gY$ of roughly the form $(h^\infty \circ g)(\theta) = (\cos \theta \sin \theta, \sin \theta)^\top$. In such a case, the embedding into $\gZ$ contains two disjoint pieces $[-3\pi/4,-\pi/4] \cup [\pi/4,3\pi/4]$.
% Empirically, either at initialization or after some training, we often observe a ``figure 8'' pattern in $\gY$ of roughly the form $(h^\infty \circ g)(\theta) = (\cos \theta, \cos \theta \sin \theta)^\top$. In such a case, the embedding into $\gZ$ contains two disjoint pieces $[-\pi/4,\pi/4] \cup [3\pi/4,5\pi/4]$ parameterized as
% \[
% (\pi \circ h^\infty \circ g)(\theta) = 
% \begin{cases}
% \theta/2 & -\pi/2 < \theta < \pi/2  \\
% 3\pi/2-\theta/2 &  \pi/2 <\theta < 3\pi/2.
% \end{cases}
% \]
One half of the circle is mapped to one piece at the top of the circle and the other half of the circle is mapped to the other disjoint piece at the bottom of circle. Such a mapping takes advantage of the singularity at $(0,0)$\footnote{In practice, we never exactly cross $(0,0)$ but only pass by it with a small distance.} into order to embed $S^1$ in two continuous pieces.  The resulting mapping is \emph{nearly} bijective, failing to reconstruct on only a small region near the two discontinuities.  
It is also mostly continuous, having only two discontinuities.  

Once this local minimum is obtained, it is very difficult to move out of it using gradient descent.  It is unlikely for the two pieces to join together and become a homeomorpic embedding since this would require passing one disjoint segment through the other and reversing its orientation which would violate bijectivity and increase the reconstruction loss.  

\begin{wrapfigure}{r}{0.4\textwidth}
    \vspace{-1.5em}
    \centering
    \includegraphics[width=0.335\textwidth]{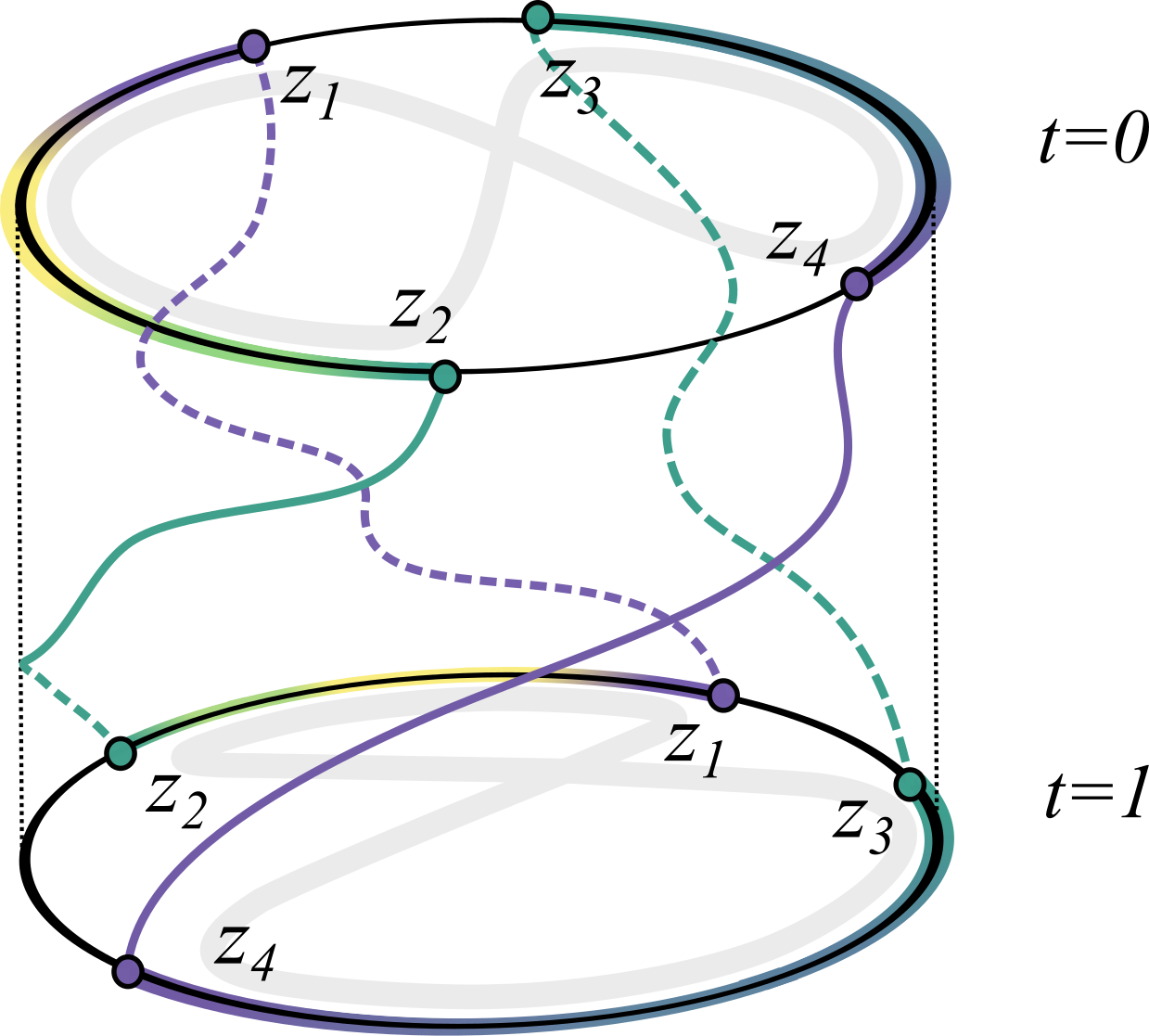}
    % \vspace{-0.25em}
    \caption{The figure 8 pattern in $\gY$ (grey) maps to two disconnected components in $\gZ$.  The cyclic order of these 4 endpoints is preserved by homotopy.  Following the parameterization of the data manifold the cyclic order is $(z_1,z_2,z_4,z_3)$, which is distinct from the cyclic order of a homeomorphic embedding, either $(z_1,z_2,z_3,z_4)$ or $(z_4,z_3,z_2,z_1)$.}
    \label{fig:Homotopy}
    \vspace{-2.0em}
\end{wrapfigure}

We make this observation precise by noting that continuous optimization preserves the ordering of points on the circle. Let $(z_1,z_2,z_3,z_4)$ denote the four end points of the two disjoint intervals of $f_{\phi(0)}$. 
%We index points based on a the parameterization of $\gM$ such that $z_i = (\pi \circ h_{\phi(0)} \circ g)(\theta_i)$ where $0 \leq \theta_1 < \theta_2 < \theta_3 < \theta_4 < 2 \pi$.  
The ordering of these points is only defined up to cyclic permutations $\mathrm{mod}\ C_4$.
%in the group $C_4$, that is $(z_1,z_2,z_3,z_4)\ \mathrm{mod}\ C_4 = (z_2,z_3,z_4,z_1)\ \mathrm{mod}\ C_4$.  Cyclic orderings of $k$ points are called necklaces and we can see there are $6$ distinct necklaces on 4 points up to equivalence. 
Proposition~\ref{prop:fig8} states that continuous optimization must preserve ordering $\mathrm{mod}\ C_4$. Figure \ref{fig:Homotopy} illustrates the proof.

\begin{proposition}
\label{prop:fig8}  
Assume that $f_{\phi(t)}$ undergoes continuous optimization.
%and is thus continuous in $t$.
Assume that $f_{\phi(t)} \circ g$ is injective for all $t$.  The cyclic ordering induced on $k$ points by $f_{\phi(0)}$ is equal to $f_{\phi(1)}$. Thus a figure 8 embedding, which corresponds to cyclic order $(z_1,z_2,z_4,z_3)\ \mathrm{mod}\ C_4$, cannot be transformed to a homeomorphic embedding, which has cyclic order $(z_1,z_2,z_3,z_4)\ \mathrm{mod}\ C_4$ or $(z_4,z_3,z_2,z_1)\ \mathrm{mod}\ C_4$.
\end{proposition}

%The figure 8 embedding has cyclic order $(z_1,z_2,z_4,z_3)\ \mathrm{mod}\ C_4$.  A homeomorphic embedding would correspond instead to the cyclic order $(z_1,z_2,z_3,z_4)\ \mathrm{mod}\ C_4$, which is not equivalent. 
In other words, transition from a figure 8 embedding to a homeomorphic embedding is impossible without violating continuity during optimization. This indicates that escaping a figure 8 local optimum during training would need to rely on discrete jumps and likely the stochasticity of the gradient estimate.

\subsection{Degree Obstructions}

A second class of topological defects that can arise are encodings with a discrepancy in the winding number. We can compute the degree, or winding number, of a map $\psi: S^1 \to S^1$ around the origin by summing up the differentials along its path on the sphere,
\begin{equation*}
    \mathrm{\omega}(\psi) = \frac{1}{2\pi} \int_{S^1} d\psi(\theta).
\end{equation*}
This concept can be extended to arbitrary connected oriented manifolds, where it is usually referred to as the degree of a mapping. Intuitively, it describes the number of times that the \emph{domain manifold} wraps around the \emph{co-domain manifold}. 

%Consider $\gM_x \cong \gZ \cong S^1$.  In this case $\pi \circ f_\phi \circ g \colon \gM \to \gZ$ is mapping $S^1 \to S^1$.  The composition $f_\phi^{\gM} = \pi \circ h_\phi \circ g \colon \gM \to \gZ$ is a mapping $S^1 \to S^1$.
% If the embedded image $h_\phi(\gM_x)$ does not contain the origin, then the mapping factors through $\mathbb{R}^2 \setminus \lbrace (0,0) \rbrace$ and is consequently continuous.  
% Denote $\pi'  = \pi |_{\mathbb{R}^2 \setminus \lbrace (0,0) \rbrace}.$ 
% We assume that $h_{\phi(t)}(x) \not = (0,0)$ for any $t,x$.
If the embedded image $h_{\phi(t)}(\gM_x)$ does not contain the origin, then the mapping factors through $\mathbb{R}^2 \setminus \lbrace (0,0) \rbrace$ and is consequently continuous. Therefore, a continuous path in $\phi$  yields a continuous path in $\gZ$.
If $f_\phi|_{\gM_x}$ is continuous, then the function $\psi_\phi := \pi \circ h_\phi \circ g: S^1 \to S^1$ has a well-defined degree, also known as winding number $w(\psi_\phi) \in \mathbb{Z}$.  In order for $f_\phi|_{\gM_x}$ to be a homeomorphism, the winding number must be $w(\psi_\phi) \in \lbrace -1, 1 \rbrace$.  Under random initialization, however, the initial network may have winding number equal to any integer. Assuming continuous optimization and continuous embeddings, %defines a homotopy $t \mapsto f_{\phi(t)}$,
then $h_{\phi(t)}(x)$ is a continuous function of both $x,t$. 
We assume that $h_{\phi(t)}(x) \not = (0,0)$ for any $t,x$ and thus winding number is defined for any time.
The following proposition thus holds.
\begin{proposition}
\label{prop:winding}
Winding numbers of the initialized and final model are equal $w(\psi_{\phi(0)}) =  w(\psi_{\phi(1)})$.
\end{proposition}
In practice, neither the continuous optimization assumption nor the avoidance of the origin holds.  Rather $h_\phi$ is updated by SGD in discrete jumps and $h_\phi \circ g$ may map to the origin.  Thus, empirically, we do see that the winding number may change during training.  However, if the initialization avoids the origin, then due to the tendency of the magnitude of the unnormalized embeddings $h_\phi(x)$ to grow during optimization  (see Section \ref{subsec:maggrowth}), the winding number changing becomes more unlikely. This means that defects in the winding number pose significant obstruction to learning homeomorphic embeddings. The winding number is also the primary optimization obstruction which makes it impractical to remove the hard projection $\pi$.  If instead we decode directly from $y \in \gY$ but push embeddings to the unit circle using the loss $|\|y\|-1|$, then it is far more likely we converge to discontinuous embeddings with the incorrect winding number (Appendix \ref{app:decodefromy}). 

\subsection{Magnitude Growth in $\gY$}
\label{subsec:maggrowth}
% \section{Magnitude Growth in $\gY$}

Empirically, we observe the values of the embeddings in $\gY := \mathbb{R}^2$ continually grow during training.  This phenomenon makes it more difficult for the embedded data manifold $h_\phi(\gM_x)$ to cross the origin and for the winding number to change.  We give a theoretical explanation for this behavior.

Consider what would happen if the embedding $y$ were updated directly based on the gradient of the loss $\nabla_y \gL$ with respect to $y$.  We assume the loss depends only on $z = y/\|y\|$,
%and thus the loss is positive scalar invariant $\gL(cy) = \gL(y)$ for $c \in \mathbb{R}_{>0}$.
and so has level sets which are unions of radial rays from the origin. % eminating from the origin.  
The gradient $\nabla_y \gL$ must then be tangent to a circle about the origin.
%be orthogonal to the level sets, it must be purely angular.  
That is, for $y = (a,b)$, the gradient $\nabla_y \gL = ( \pm b,\mp a)$.  Under gradient flow, the evolution of $y$ in time $y_t$ would thus flow along circles of fixed radius and so
%thus the $y$ embeddings would have fixed radius
$\| y_0 \| = \| y_t \|$.  
Under gradient descent, however, due to the convexity of the flow lines, which are circular, the embeddings $y$ will tend to grow in magnitude.  For $\eta \in \mathbb{R}_{>0}$, we compute
\begin{align*}
    \| y - \eta \nabla_y \gL \|^2 = (a \mp \eta b)^2 + (b \pm \eta a)^2  
    = (a^2 + b^2)(1 + \eta^2) 
    > \|y \|^2.
\end{align*}

In practice, however, we do not update $y$ based on $\nabla_y \gL$ but rather based on the gradient with respect to model parameters $\phi$.  Let $F \colon \Phi \to \gY$ be the map from model parameters $\phi$ to $y$ given fixed input data $x$.  Then the actual update to $y$ is $\tilde{\nabla}_y \gL = dF^T \circ \nabla_\phi \gL$ where $dF$ is the total derivative or Jacobian of the map $F$.  Since $\nabla_\phi \gL = (dF) \nabla_y \gL$ we have $\tilde{\nabla}_y \gL = dF^T dF \nabla_y \gL$.  The angle between $\tilde{\nabla}_y \gL$ and $\nabla_y \gL$ is bounded by some $\theta$ a quantity depending on the eigenvalues of the operator $dF$.  Given that $\nabla_y \gL$ is tangential to the circle, assuming for simplicity $\tilde{\nabla}_y \gL$ has constant length $L$ and uniform distribution $[-\theta,\theta]$ in angle to $\nabla_y \gL$, the norm of $y$ still grows \emph{in expectation}. 

\begin{proposition}
Assume a circle of radius $R$.  Let $v$ be a random vector at $y$ on the circle of length $L$ with angle to the tangent uniform in $[-\theta,\theta]$.  Then
\[
\mathbb{E}[\|y + v\|^2] = L^2+R^2 > R^2 = \| y \|^2.
\]
\end{proposition}

%\vspace{-0.5\baselineskip}
\section{GroupFlow-VAE}
\label{sec:gf-vae}
%\vspace{-0.5\baselineskip}

\begin{figure}[!t]
\begin{center}
\includegraphics[width=0.95\textwidth]{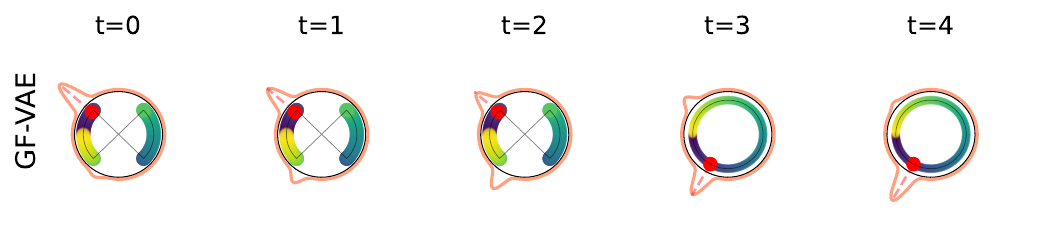}
\end{center}
\vspace*{-3.0ex}
\caption{
% Intuition for how a VAE with a multimodal encoder can \emph{untwist} a figure 8 defect by switching modes. 
The initialized encoder (t=0) may contain defects as described in Section~\ref{sec:theory}, such as a figure 8 pattern.
With a standard VAE, where the learned representation (red point) is the mean of a conditional gaussian, these defects are unlikely to be resolved during optimization, as it would require passing through intermediate parameters (and hence representations) with higher reconstruction loss. Using a multimodal variational distibution, the parameters and corresponding representations are less tightly coupled, in the sense that a continuous change in the parameter space can result in a discontinues change in representation (i.e. the mode of the distribution) without passing through \emph{high-loss} areas of the parameters space.
}
% \caption{Intuition for how a VAE with a multimodal encoder can \emph{untwist} a figure 8 defect by switching modes. We show the multimodal density corresponding for a single point along with the trajectory in the intermediate, with mode shown in red. We also show the trajectory in the intermediate space $\gY$ corresponding to the mode for each point on the data manifold $\gM_x$. In a multimodal distribution, the amplitude of each mode can be adjusted under continuous optimization, which will make it possible to switch from a mode that contains a defect to a mode that does not contain a defect without passing through intermediate parameter values that will increase the reconstruction loss.}
%with the mode highlighted in red. Between time-steps 2 and 3, the representation jumps from the upper-left area to the lower-left by shifting most of the density to that area. By doing this for all other points, a homeomorphic mapping can be achieved.  
% In a GF-VAE there exists a continuous path from figure 8 shape to a circle which is simply moving mass from one area to the other.

%\vspace{-3.5ex}
\label{fig:untwist}
\end{figure}

The topological obstructions to optimization that we identified in Section~\ref{sec:theory} cannot be resolved under continuous optimization. Moreover, even if we allow for a degree of discontinuity, resolving obstructions like a figure 8 will require violating bijectivity, which implies that the reconstruction loss must increase to escape this defect.
%defect are difficult to resolve, since "untwisting" the trajectory in the intermediate space $\gY$ induces a reordering in the latent space $\gZ$. 
This suggests that a VAE with a reparameterized construction as described in Equation~\ref{eq:reparam-sampling-lie-group} will be susceptible to local optima, which aligns with the empirical observation that homeomorphic VAEs can be difficult to train.

This raises the question of whether we can make VAEs less susceptible to topological obstructions by employing a different parameterization from that of Equation~\ref{eq:reparam-sampling-lie-group}. More concretely, we will consider a parameterization that admits multiple modes in the variational distribution $q_\phi(z|x)$, rather than the unimodal construction in Equation~\ref{eq:reparam-sampling-lie-group}, and define $f_\phi(x)$ in terms of the mode,
\begin{equation}
    f_\phi(x) := \argmax_{z} q_\phi(z|x).
\end{equation}
The intuition behind this approach is illustrated in Figure~\ref{fig:untwist}. If the variational distribution contains multiple modes, rather than a single peak centered at $\pi(h_\phi(x))$, then it may becomes possible for small changes to result in non-local changes, such as the reordering of points that is needed to untwist a figure 8, by switching between modes in the variational distribution. Moreover, increasing the number of parameters of $q_\phi(z|x)$ is likely beneficial for escaping some of the topological obstructions as some defects such as self-intersection are less likely to occur in high-dimensional spaces.

Motivated by this intuition, we consider a parameterization of $q_\phi(z|x)$ that employs normalizing flows~\citep{rezende2015variational}. In a normalizing flow, the sample $z$ is defined as a push forward of sequence of smooth bijective transformations, which makes it possible to reshape a simple unimodal distribution into a more complex multimodal distribution.
%Motivated by these problems, we propose making the following change: Rather than employing a Gaussian parameterization, we parameterize $q_\phi(z|x)$ using normalizing flows to define a family of flexible multimodal distributions and define the final representation for which we wish to learn a homeomorphic mapping as the mode of this distribution $f_\phi(x) := \argmax_{z} q_\phi(z|x)$. Normalizing flows~\citep{rezende2015variational} are a powerful class of models in machine learning that transform simple probability distributions into a complex distribution by applying a sequence of bijective , smooth transformations. 
% The name ``normalizing flows'' comes from the fact that the transformations are designed to preserve the normalization of the probability density function.
The probability of a sample from the final density can be computed by repeatedly applying the rule for change of variables. Concretely, given a base distribution $p(z_0)$ and a sequences of bijective transformations $r_k: \gZ \rightarrow \gZ$, we obtain a sample $z=z_K$ and probability by first sampling $z_0 \sim p(z_0)$ and defining a sequence of transformations
\begin{equation}
\label{eq:flow}
    z_{k} = r_k(z_{k-1}), \qquad \log p(z_k) = \log p(z_{k-1}) - \log \left| \det \frac{\partial \ r_k(z_{k-1})}{\partial \ z_{k-1}} \right| \qquad \text{for } k = 1\cdots K.
\end{equation}
Normalizing flows have been used to define distributions on geometric structures such as Lie groups by either defining a flow on the Lie algebra and computing the push-forward density of the exponential map~\citep{falorsi2019reparameterizing}, or designing structure-specific transformations~\citep{rezende2020normalizing}. 

In the GF-VAE, we will define a construction in which the encoder network returns the parameters of the flow. We show an overview of the architecture in Figure~\ref{fig:overview}. We define a network $h_{\phi}: \gX \to \gY := \sR^{K \times l}$ where $K$ and $l$ are the number of flow layers and  parameters respectively, a sequence of bijective transformations $\{r(\cdot ; y_{k})\}_{k=1}^{K}$ parameterized by $\{y_k\}_{k=1}^{K}$, and a base distribution distribution $p(z_0)$ which we define as a distribution on the group. The sequence $\{r(\cdot ; y_{k})\}_{k=1}^{K}$ is then used to define a new distribution on $\gZ$ given an $x$ by transforming the base distribution. This defines a conditional flow that can be trained using a stand lower bound (Eq.~\ref{eq:elbo}),
\begin{equation*}
\label{eq:gf-sample}
    z_0 \sim p(z_0), \qquad z_k = r(z_{k-1}; y_{k}), \qquad \log q(z_k | x) = \log q(z_{k-1} | x) - \log \left| \det \frac{\partial\ r(z_{k-1}; y_k) }{\partial\ z_{k-1}} \right|.
\end{equation*}
% :
% \begin{equation}
% \label{eq:gf-vae}
%     \gL_{\phi,\theta}^{\text{GF-VAE}}(x) = \E_{z_K \sim q_\phi(z_K|x)} \left[ \log p_\theta(x,z_K) \right] + \sH \left[ q_\phi(z_K|x) \right]
% \end{equation}

% Note that the mode is not used explicitly during training, but only measured after training. 

The choice of the flow $r$ is very important here as normalizing flows are typically defined on flat spaces. This means that, for a specific manifold $\gM$, additional care must be taken when designing them to ensure that they are a diffeomorphism from $\gZ$ to itself~\citep{rezende2020normalizing,durkan2019neural,mathieu2020riemannian}. In this paper, we employ a similar method as~\citeauthor{falorsi2019reparameterizing} where we define an affine layer followed by a single layer of spline flow, and a Tanh layer multiplied by $\pi$ as the last layer to push all the probability density between $-\pi$ and $\pi$. 
% {\color{red} DOUBLE CHECK IN CODE In this paper, we employ the method from~\citeauthor{rezende2020normalizing} that defines diffeomorphic functions on circles and spheres. In the case of the circle $\gS^1$, the conditions for flow $r: \left[-\pi, \pi \right] \rightarrow \left[-\pi, \pi \right]$ are $r(-\pi) = -\pi$, $r(\pi) = \pi$, $\partial r(z) / \partial z > 0$, and  $\partial r(z) / \partial z |_{z=-\pi} = \partial r(z) / \partial z |_{z=\pi}$. As in~\citeauthor{rezende2020normalizing}, we use neural spline flows~\citep{durkan2019neural} to enforce these constraints which can be done relatively easily by imposing additional constraints on the knots.}
% Generally speaking, we impose no additional constraint on the flow besides the Lie group itself so we can use any flow designed for $\gZ$.
% enforcing the first and the last knot to be $(-pi, -pi)$ and $(pi, pi)$ respectively as well enforcing the gradients for these to be the same. 

% Inspired by this property, we propose to use normalizing flows which are capable of modelling highly flexible multimodal distributions as the variational distribution on the latent space. Concretely, we choose a base distribution $p(z_0)$ which we set to be uniform a distribution on the group, and $K$ transformations $r_{y_{k}}$ parameterized by $y_k \in \sR^{d_r}$ where $d_r$ is the dimensionality of the parameter space of the flow and $y\in \sR^{K \times d_r}$ is the output of the network $h_\phi$ (Figure~\ref{fig:overview}). To put it simply, the network $h_\phi$ is not parameterizing a flow rather than a Gaussian.

\begin{figure}[!t]
\centering
\vspace*{-1.0ex}
\includegraphics[width=\textwidth]{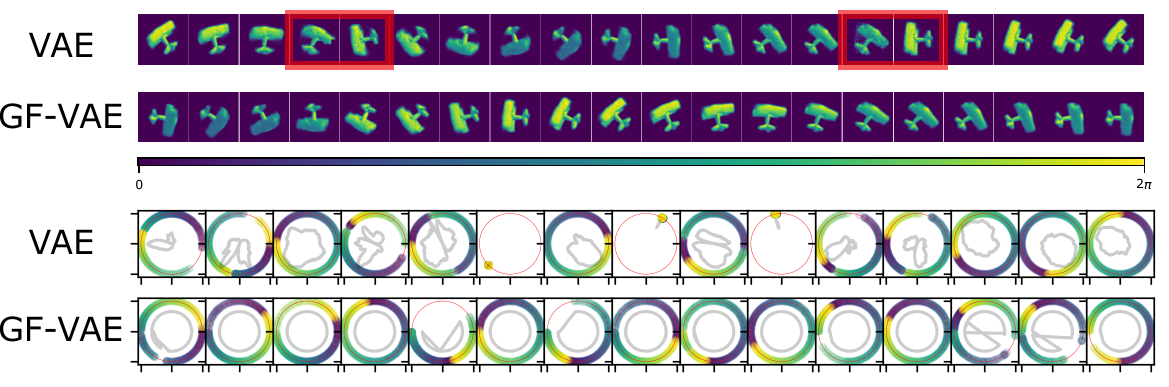}
\vspace*{-1.45ex}
\caption{\emph{Top}: Latent traversals in the decoder for a VAE and GF-VAE after training. \emph{Bottom:} The representations in latent space $\gZ$ for VAEs and GF-VAEs initialized with 15 different random seeds. For VAEs, the gray line inside the circle show the (scaled-down) $y$-traversals. For GF-VAEs, we it is difficult to visualize $\gY$ space as it is high-dimensional space. Therefore, to inspect for obstructions, we show the traversal of the $z$ vectors instead. }
% \caption{{\color{red} TODO: PLACE HOLDER; UPDATE FIGURE FOR GF-VAE RESULT. Latent space interpolation in the decoder (\emph{Top}) and encoder (\emph{Bottom}).}}
\vspace*{-1.25ex}
\label{fig:interpolation}
\end{figure}

\begin{table}[!b]
\caption{Number of learned homeomorphic encoders and their continuity score for different objectives trained with 15 different random seeds.}
\label{tab:image-metrics}
\centering
\small
\begin{tabular}{lll|ll|ll}
& \multicolumn{2}{c}{\emph{L-shaped Tetromino}} & \multicolumn{2}{c}{\emph{Teapot}} & \multicolumn{2}{c}{\emph{Airplane}} \\ 
                    & \# H.       & Continuity             & \# H.       & Continuity           & \# H.       & Continuity          \\
\midrule
AE		                     & 2/15		 	 & $137.28		\pm 54.15$ 	 & 	 0/15		 	 & $173.37		\pm 21.99$ 	 &	0/15		 	 & $132.95		\pm 40.61$	 \\
VAE $(\beta = 1)$		     & 0/15		 	 & $117.02		\pm 23.17$ 	 & 	 6/15		 	 & $14.21		\pm 7.30$ 	 &	0/15		 	 & $86.72		\pm 65.12$ \\
VAE $(\beta = 4)$		     & 2/15		 	 & $132.45		\pm 58.79$ 	 & 	 \textbf{15/15}& $\boldsymbol{1.22		\pm 0.05}$	& 5/15		 	 & $67.88		\pm 59.20$	 \\
VAE + $y$-reg & 1/15 & $152.71 \pm 74.20$ &  1/15 & $14.51 \pm 81.$ & 3/15 & $124.10 \pm 93.53$ \\
GF-VAE $(\beta = 1)$		 & 5/15		 	 & $50.51		\pm 58.18$ 	 & 	 7/15		 	 & $21.07		\pm 24.42$ 	 &	0/15		 	 & $63.94		\pm 34.43$ 	 \\
GF-VAE $(\beta = 4)$		 & 9/15		 	 & $24.04		\pm 29.14$ 	 & 	 13/15		 & $8.16		\pm 17.89$ 	 &	7/15		 	 & $\boldsymbol{27.70		\pm 29.23}$	 \\
Action-GF-VAE $(\beta = 4)$	 & \textbf{10/15}	 	 & $\boldsymbol{19.35		\pm 23.37}$ 	 & 	13/15		 & $3.14		\pm 1.74$ 	 &	\textbf{9/15}	 & $33.26		\pm 40.03$ \\
% AE		 & 3/15		 & $141.28\pm 39.$		& 0/15			 & $48.59\pm 7.11$		 & 2/15		 & $104.94\pm 44.$		  \\
% VAE & 0/15 & $133.83 \pm 26.$ & 3/15 & $18.59\pm 8.02$ & 0/15 & $91.40 \pm 31.$ \\
% {\color{red} $\beta$-VAE} & & & & & & \\
% VAE + $y$-reg & 1/15 & $152.71 \pm 74.20$ &  1/15 & $14.51 \pm 81.$ & 3/15 & $124.10 \pm 93.53$ \\
% GF-VAE ($\beta=1$) & 8/15 & $26.68 \pm 37.04$ &  7/15 & $35.99 \pm 51.33$ & 9/15 & $35.23 \pm 41.71$ \\
% GF-VAE ($\beta=4$) & \textbf{14/15} & \textbf{7.90 $\pm$ 13.01} &  \textbf{15/15} & \textbf{2.79 $\pm$ 1.36} & \textbf{12/15} & \textbf{16.88 $\pm$ 12.07} \\
\midrule
Sup-VAE		 & \textbf{15/15}		 & \textbf{5.74$\pm$0.49}	 & 13/15	 & 5.67$\pm$ 0.34		 &	0/15	 & $96.78\pm 38.$		 	 \\
\end{tabular}
% VAE $(\beta=1)$		 & 3/15		 & 15/15		 & 12/15	 & $18.59\pm 8.02$		 & $29.83\pm 5.44$		 & $7.68\pm 0.08$		 \\

\end{table}

\vspace{-0.5\baselineskip}
\section{Experiments}\label{sec:exp}
\vspace{-0.25\baselineskip}

We perform a series of experiments to evaluate the difficulty of learning a homeomorphic encoder. Concretely, we investigate how often we fall into one of the failure cases described in Section~\ref{sec:theory} in standard VAEs with geometric latent spaces in practice. Subsequently, we examine how well GF-VAE can circumvent these topological obstructions during training. 
% we examine whether training with the Isom-AE objective helps in terms of learning an optimal encoder.

\paragraph{Baselines.} Throughout our experiments, we compare against (1) a standard VAE, and (2) a supervised VAE where in addition to maximizing the ELBO, the encoder is trained to predict the ground truth representations. These two scenarios serve as extremes on the spectrum of guiding the model to the right representation. 
We also compare against a deterministic autoencoder (AE) which does not regularize the embedding space.
We also tried regularizing the y-space to be close $S_1$ (by penalizing $(\|y\|_2 - 1)^2$) in order to mitigate the optimization problem discussed in Subsection~\ref{subsec:maggrowth}, which we refer to as ``reg-$y$'' loss. 
Finally, we evaluate the $\beta$-VAE objective~\citep{higgins2017betavae} which increases the regularization on the latent space by upweighting the KL term in Eq~\ref{eq:elbo}. 
All models employ a 4-layer CNN architecture for the encoder and decoder with LeakyReLU activations. For the decoder, we also experiment with \emph{action-decoder} proposed by~\citeauthor{falorsi_explorations_2018}, which we found helpful for learning a homeomorphic mapping. The action-decoder uses a special first layer where the group action is applied to a set of learned Fourier coefficients rather than directly being passed as input to the architecture. For further details regarding our experiments, please refer to Appendix~\ref{app:sec:exp}.

\paragraph{Evaluation.} We evaluate all models based on two criteria: (1) Has the encoder learned a homeomorphic mapping? and (2) Has the decoder learned a good model of the data? Assessing whether a learned mapping is homeomorphic is challenging. To verify homeomorphism, we follow the evaluation proposed in~\cite{falorsi_explorations_2018} by examining whether the encoder yields a continuous path when interpolating in the data manifold from $-\pi$ to $\pi$. Details on evaluating continuity are provided in Appendix~\ref{app:sec:cont}. To determine how often the models encounter the topological obstructions described in Section~\ref{sec:theory}, we also report crossing and winding numbers in Appendix~\ref{app:sec:additional-results}. 
A crossing number grater than 0 implies a ``figure 8'' obstruction, and a winding number that is not equal to $1$ or $-1$ implies winding number obstruction. 
% For evaluating isometry, we simply compute and report the term $L_p(N)$ described in Definition~\ref{def:app-isom} for $p=2$ and $N=10000$ rather than applying a binary threshold. This enables us to compare how close to isometric different encoders are.
Lastly, we measure the log-likelihood to assess how well each model approximates the data manifold.
If the model has diverged during training due to posterior collapse, we report it as a non-homeomorphic mapping and ignore its continuity score in the average. 
% to check various necessary conditions such as if it has (1) winding number 1 or -1, (2) crossing number 0, (3) 

% \paragraph{Evaluation.} We evaluate all models based on three criteria (1) homeomorphism, (2) isomtery, and (3) reconstructions. Similar to isometry, it is practically difficult to determine if a learned mapping is homeomorphic. Here, we verify homeomorphism by examining the encoder to check various necessary conditions such as if it has (1) winding number 1 or -1, (2) crossing number 0, (3) yields a continuous path when interpolating in the data manifold from $-\pi$ to $\pi$. We define continuity in Section~\ref{app:sec:cont}. For evaluating isometry, we simply compute and report the term $L_p(N)$ described in Definition~\ref{def:app-isom} for $p=2$ and $N=10000$ rather than applying a binary threshold. This enables us to compare how close to isometric different encoders are. Lastly, we report the reconstruction error to measure how well the model approximates the data manifold. 
\vspace{-0.5\baselineskip}
\subsection{Images: SO(2)}
\vspace{-0.5\baselineskip}

In our first experiment, we train on images of an L-shaped tetromino~\citep{bozkurt2021rate}, a teapot, and an airplane~\citep{shape_benchmark}. The $\mathrm{SO}(2)$ manifold corresponding to each object is made by rotating the image of the object around the center. We report our findings in Tables~\ref{tab:image-metrics} and~\ref{app:tab:image-metrics}. 
% In order to train Isom-AE, we generate sequences of length with the same setting as the previous experiment.
% We evaluate various objectives in terms of the overall number of learned homeomorphic encoders and log-likelihood.
% We evaluate various objectives in terms of the overall number of learned homeomorphic encoders, isometry, and reconstrcutions.
% such as weight regularization, latent denoising, inadditona to AE and Sup-AE. 

%\vspace{-0.5\baselineskip}

We observe that even though the types of obstructions vary across images, both VAE and AE in general fail to learn a homeomorphic encoder. The GF-VAE objective improves performance noticeably across both metrics. We also did not find $y$ regularization to be very helpful as it mainly stabilized training towards whatever mapping that was achieved at the early stages of training. We observed the main failure case for most of the non-homeomorphic encoders was due to discontinuity emerging from figure-eight obstruction. Moreover, looking at the learning curves in Figure~\ref{app:fig:winding-numbers}, we observe that the winding number is susceptible to change during training which suggests that the continuity assumption is of importance in proposition~\ref{prop:winding}.
GF-VAE resolves both issues, which allows us to interpolate nicely in the $\gZ$-space (Figure~\ref{fig:interpolation}). What is very surprising is that in the case of airplanes, we see that even supervised objective fails to overcome these optimization obstructions, while a GF-VAE is able to achieve this at a much better rate. 

We also observe that increasing the $\beta$ value for the KL term generally helps. This is perhaps unsurprising given that a high $\beta$ value encourages the latent space to cover the prior and therefore discourages the winding numbers from being 0 and improves the performance on the continuity metric. In the case of the teapots, we in fact observe that increasing $\beta$ is sufficient to learn a homeomorphic encoder. However, GF-VAE still generally scores better in terms of continuity and winding numbers.

\begin{figure*}
    \centering
    \includegraphics[width=0.9\textwidth]{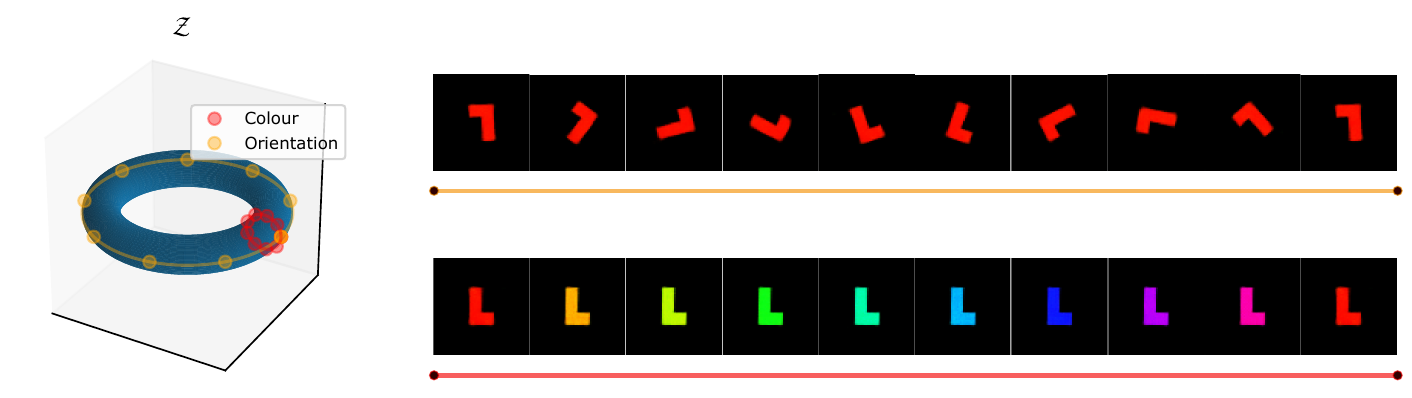}
    \vspace*{-2.25ex}
    \caption{Latent traversals in the decoder of a GF-VAE trained on Tetrominoes with a torus latent spaces. In each traversal, we keep the value for one Lie group fixed and do a geodesic interpolation across the other group. The traversals in the latent space highlighted are by red and orange. }
    \vspace*{-1.25ex}
    \label{fig:torus-result}
\end{figure*}

\vspace{-0.5\baselineskip}
\subsection{Images: Torus}
\vspace{-0.5\baselineskip}

In this experiment, we consider a ring torus as the latent space, which is homeomorphic to the Lie group $\mathrm{SO}(2) \times \mathrm{SO}(2)$. We create a dataset homeomorphic to this group by independently rotating the L-shaped tetromino in both orientation and color. All models can be extended to Tori by simply duplicating the latent space and learning multiple encodings to $\mathrm{SO}(2)$ in parallel. This is done by defining a factorized density $q_\phi(z^{(1)}, z^{(2)}|x) = q_\phi(z^{(1)}|x)q_\phi(z^{(2)}|x)$, where each distribution $q_\phi(z^{(i)}|x)$ defines a distribution on a $S^1$. Ideally, we want one subgroup to correspond to colour and the other to orientation. In this setting, we evaluate homeomorphism by picking 10 different values in either colour or orientation and measuring the continuity of the encoded path when interpolating in the data manifold on the other attribute. We identify the encoder as homeomorphic if the average continuity score of all 10 paths is below a certain threshold (see Appendix~\ref{app:sec:cont}). Unsurprisingly, we found that learning a homeomorphic mapping in tori is more challenging compared to circles. Out of 15 runs, the VAE models learn a homeomorphic mapping 1 time, while a GF-VAE manages to learn a homeomorphic mapping 7 times. To be able to align each Lie group with the corresponding attribute, we also used the weakly supervised proposed by~\citeauthor{locatello2020disentangling}, where the model receives a sequence of images in which only the angle or the colour is changing. We show the latent traversal in the decoder model for one of the successful runs of GF-VAE in Figure~\ref{fig:torus-result}. As we can see, the model has successfully managed to disentangle colour and orientation in the latent space. 
% {\color{red} mention the disentangle loss in the text}

%\vspace{-0.5\baselineskip}
%%%%%%%%%%%%%%%%%%%%%%%%%%%%%%
\section{Related Work}
\label{sec:related-work}
\vspace{-0.25\baselineskip}

\paragraph{Learning Geometric Representations.} There has been a large amount of work concerned with learning representations from data with geometric structure in the unsupervised or weakly supervised setting. %To this end, different methodologies have been brought forward. 
One line of work has focused on the use of geometric spaces such as lie groups in latent space~\citep{davidson_hyperspherical_2018,falorsi_explorations_2018,rey_diffusion_2019,vadgama2022kendall}. In ~\citep{miaoincorporating}, it was argued that that geometric inductive biases in VAEs should be Incorporated via a deterministic mapping rather than the prior which is consistent with our work. As we show in this work, naively incorporating a geometric bias leads to topological obstructions. Another line of work has focused on inferring the latent structure by exploring the local connectivity information~\citep{moor2020topological, xingyu2021neighborhood,lee2021neighborhood,pfau_disentangling_nodate}.

%  None of the litaerutre have not discussed nor use normalizing flows as the inference distribution.

% \paragraph{Unsupervised or Weakly Supervised Learning of Geometric Representations}
% There has been a large amount of work concerned with learning representations from data with geometric structure in the unsupervised, weakly supervised, or semi-supervised setting. 
% To this end, different methodologies have been brought forward, including the use of 
% hyperspherical prior distributions \citep{davidson_hyperspherical_2018}, 
% specializations of the reparameterization trick ~\citep{falorsi_explorations_2018, rey_diffusion_2019}, and making use of local connectivity information 
% \citep{moor2020topological, xingyu2021neighborhood,lee2021neighborhood}.

\paragraph{Learning Disentangled Representations.}
Topological group structure can be used to define a notion of disentanglement that is based on equivariant properties under group transformations~\citep{higgins_towards_2018}.
There exists a  body of work that aims to learn both disentangled and equivariant group representations.
One set of methods relies on agent actions to predict the group element~\citep{caselles-dupre_symmetry-based_2019,quessard2020learning}. 
An adjacent line of work is to regularize encoders to be equivariant with respect to the group action, using triplets of the form $(x_t,m_t,x_{t+1})$ or longer sequences~\citep{guo2019affine,dupont2020equivariant,pmlr-v162-tonnaer22a}. 
% The objective proposed in~\citep{pmlr-v162-tonnaer22a} is similar to Isom-VAE with the difference being that they assumes knowledge of ground truth group elements while we only assumes that sequences are generated by repeated application of the same unknown group element. 
Some approaches focus on learning symmetry-based disentangled representations in fully unsupervised settings or by enforcing commutativity in the latent Lie group~\citep{yang2022towards,zhu_commutative_2021}. Another related area disentangles \emph{class} and \emph{pose} in the latent space~\citep{marchetti2022equivariant,winter2022unsupervised}. Our work does not focus on disentanglement, but the topological obstructions that we describe are relevant to this domain.

\paragraph{Topological Obstructions in Learning.}
% To develop a better understanding of the failure modes that we observe in Figure~\ref{fig:so2-failed-cases}, we will formalize topological obstructions to training homeomorphic embeddings in Section \ref{sec:theory}.
The topological obstructions that we consider in this work are distinct from the homological obstructions that have been characterized in prior work~\citep{de_haan_topological_2018,batson2021topological,falorsi_explorations_2018,rey_diffusion_2019}.
Theorem 1 by \citet{de_haan_topological_2018} defines homological obstructions as follows:
%; for
%end formalized by \c; 
For any latent space $\gZ$ with non-trivial topology, it is possible to learn an encoder $f_\phi$ that is continuous when restricted to $\gM_x \subset \gX$, but this encoder must be discontinuous on the full space $\gX$. For this reason, \citet{falorsi_explorations_2018} and others \citep{xu2018spherical, meng2019spherical} use the two-part encoder from \eqref{eqn:fphi}, inserting discontinuous layers $\pi$ when mapping to circles, spheres, $\mathrm{SO}(n)$, or other manifolds.  This explicit discontinuity circumvents the homological obstruction without forcing the linear layers of the network to approximate discontinuities using large weights, which we and others find leads to instability during training and inferior reconstructions (Section \ref{sec:exp}).  

\paragraph{Normalizing Flows on Manifolds.} In recent years, there has been a surge of interest in extending normalizing flows, originally formulated for Euclidean spaces, to Riemannian manifolds~\citep{rezende2020normalizing,mathieu2020riemannian,kohler2021smooth,durkan2019neural}. One approach involves leveraging the Lie group structure of the manifold to define a parametrization of the flow~\citep{rezende2020normalizing}, which is also the strategy we adopt in this work. These recent advancements have paved the way for applying normalizing flows to the group $\mathrm{SO}(3)$ in order to learn pose estimation in molecular structures~\citep{kohler2023rigid} and images~\citep{liu2023delving,pmlr-v139-murphy21a}. Our work differs from these efforts in that the primary objective of our method is not to target flexible distributions on Riemannian manifolds, but rather to demonstrate that utilizing a flow as a variational distribution aids \emph{optimization} in learning a homeomorphic embedding.

%%%%%%%%%%%%%%%%%%%%%%%%%%%%%%%%%%%%%%%%%%%%%%%

%%%%%%%%%%%%%%%%%%%%%%%%%%%%%%%%%%%%
\vspace{-0.25\baselineskip}
\section{Conclusion}
\label{sec:conclusion}
\vspace{-0.25\baselineskip}
%%%%%%%%%%%%%%%%%%%%%%%%%%%%%%%%%%%%

%\vspace{-0.5\baselineskip}

In this paper, we investigate obstructions to optimization that can arise when learning encoders for topological spaces.  We classify different types of obstructions, provide evidence these are encountered in practice, and give mathematical explanations for how they occur under certain assumptions. 
%We find, surprisingly, that these obstruction make learning homeomorphic mappings difficult even in low-dimensional data or with supervision.
We propose GF-VAE, a novel model that employs normalizing flows as variational distributions to help circumvent these issues and show its effectiveness across several datasets when encoding to circles and tori. This work contains several limitations.  Firstly, our theoretical analysis is limited by the idealized assumptions necessary to analyze the method using topological tools which do not exactly match those encountered in practice. 
%Secondly, Isom-AE requires $\cM$ to be a group manifold and would need to be modified to work on, e.g., higher dimensional spheres.  
Secondly, the metrics we define such as winding and crossing numbers are harder to define and compute for higher dimensional manifolds. Future work includes expanding our analysis and techniques to a wider array of Lie groups and non-group manifolds.

\begin{ack}
This work was supported by NSF grants \#2107256 and \#2134178. We also would like to thank Patrick Forré, Sharvaree Vadgama, Marco Federici, and Erik Bekker for helpful discussions. 

% Use unnumbered first level headings for the acknowledgments. All acknowledgments
% More information about this disclosure can be found at: \url{https://neurips.cc/Conferences/2023/PaperInformation/FundingDisclosure}.

% Do {\bf not} include this section in the anonymized submission, only in the final paper. You can use the \texttt{ack} environment provided in the style file to autmoatically hide this section in the anonymized submission.
\end{ack}

\bibliography{references}

\begin{thebibliography}{48}
\providecommand{\natexlab}[1]{#1}
\providecommand{\url}[1]{\texttt{#1}}
\expandafter\ifx\csname urlstyle\endcsname\relax
  \providecommand{\doi}[1]{doi: #1}\else
  \providecommand{\doi}{doi: \begingroup \urlstyle{rm}\Url}\fi

\bibitem[Batson et~al.(2021)Batson, Haaf, Kahn, and
  Roberts]{batson2021topological}
Joshua Batson, C~Grace Haaf, Yonatan Kahn, and Daniel~A Roberts.
\newblock Topological obstructions to autoencoding.
\newblock \emph{Journal of High Energy Physics}, 2021\penalty0 (4):\penalty0
  1--43, 2021.

\bibitem[Bengio et~al.(2013)Bengio, Courville, and
  Vincent]{bengio2013representation}
Yoshua Bengio, Aaron Courville, and Pascal Vincent.
\newblock Representation learning: A review and new perspectives.
\newblock \emph{IEEE transactions on pattern analysis and machine
  intelligence}, 35\penalty0 (8):\penalty0 1798--1828, 2013.

\bibitem[Bozkurt et~al.(2021)Bozkurt, Esmaeili, Tristan, Brooks, Dy, and van~de
  Meent]{bozkurt2021rate}
Alican Bozkurt, Babak Esmaeili, Jean-Baptiste Tristan, Dana Brooks, Jennifer
  Dy, and Jan-Willem van~de Meent.
\newblock Rate-regularization and generalization in variational autoencoders.
\newblock In \emph{International Conference on Artificial Intelligence and
  Statistics}, pages 3880--3888. PMLR, 2021.

\bibitem[Bronstein et~al.(2021)Bronstein, Bruna, Cohen, and
  Veli{\v{c}}kovi{\'c}]{bronstein2021geometric}
Michael~M Bronstein, Joan Bruna, Taco Cohen, and Petar Veli{\v{c}}kovi{\'c}.
\newblock Geometric deep learning: Grids, groups, graphs, geodesics, and
  gauges.
\newblock \emph{arXiv preprint arXiv:2104.13478}, 2021.

\bibitem[Caselles-Dupr{\'e} et~al.(2019)Caselles-Dupr{\'e}, Garcia~Ortiz, and
  Filliat]{caselles-dupre_symmetry-based_2019}
Hugo Caselles-Dupr{\'e}, Michael Garcia~Ortiz, and David Filliat.
\newblock Symmetry-based disentangled representation learning requires
  interaction with environments.
\newblock \emph{Advances in Neural Information Processing Systems}, 32, 2019.

\bibitem[Chen et~al.(2021)Chen, Wang, Lan, Zheng, and
  Zeng]{xingyu2021neighborhood}
Xingyu Chen, Chunyu Wang, Xuguang Lan, Nanning Zheng, and Wenjun Zeng.
\newblock Neighborhood geometric structure-preserving variational autoencoder
  for smooth and bounded data sources.
\newblock \emph{IEEE Transactions on Neural Networks and Learning Systems},
  pages 1--14, 2021.
\newblock \doi{10.1109/TNNLS.2021.3053591}.

\bibitem[Cohen and Welling(2016)]{cohen2016group}
Taco Cohen and Max Welling.
\newblock Group equivariant convolutional networks.
\newblock In \emph{International conference on machine learning}, pages
  2990--2999. PMLR, 2016.

\bibitem[Davidson et~al.(2018)Davidson, Falorsi, Cao, Kipf, and
  Tomczak]{davidson_hyperspherical_2018}
Tim~R. Davidson, Luca Falorsi, Nicola~De Cao, Thomas Kipf, and Jakub~M.
  Tomczak.
\newblock Hyperspherical {Variational} {Auto}-{Encoders}.
\newblock In Amir Globerson and Ricardo Silva, editors, \emph{Proceedings of
  the {Thirty}-{Fourth} {Conference} on {Uncertainty} in {Artificial}
  {Intelligence}, {UAI} 2018, {Monterey}, {California}, {USA}, {August} 6-10,
  2018}, pages 856--865. AUAI Press, 2018.
\newblock URL \url{http://auai.org/uai2018/proceedings/papers/309.pdf}.

\bibitem[de~Haan and Falorsi(2018)]{de_haan_topological_2018}
Pim de~Haan and Luca Falorsi.
\newblock Topological {Constraints} on {Homeomorphic} {Auto}-{Encoding}.
\newblock In \emph{NIPS Workshop}, December 2018.
\newblock URL \url{http://arxiv.org/abs/1812.10783}.
\newblock arXiv: 1812.10783.

\bibitem[Dupont et~al.(2020)Dupont, Martin, Colburn, Sankar, Susskind, and
  Shan]{dupont2020equivariant}
Emilien Dupont, Miguel~Bautista Martin, Alex Colburn, Aditya Sankar, Josh
  Susskind, and Qi~Shan.
\newblock Equivariant neural rendering.
\newblock In \emph{International Conference on Machine Learning}, pages
  2761--2770. PMLR, 2020.

\bibitem[Durkan et~al.(2019)Durkan, Bekasov, Murray, and
  Papamakarios]{durkan2019neural}
Conor Durkan, Artur Bekasov, Iain Murray, and George Papamakarios.
\newblock Neural spline flows.
\newblock \emph{Advances in neural information processing systems}, 32, 2019.

\bibitem[Falorsi et~al.(2018)Falorsi, de~Haan, Davidson, De~Cao, Weiler,
  Forré, and Cohen]{falorsi_explorations_2018}
Luca Falorsi, Pim de~Haan, Tim~R. Davidson, Nicola De~Cao, Maurice Weiler,
  Patrick Forré, and Taco~S. Cohen.
\newblock Explorations in {Homeomorphic} {Variational} {Auto}-{Encoding}.
\newblock In \emph{ICML 2018 workshop on Theoretical Foundations and
  Applications of Deep Generative Models}, 2018.
\newblock URL \url{http://arxiv.org/abs/1807.04689}.
\newblock arXiv: 1807.04689.

\bibitem[Falorsi et~al.(2019)Falorsi, de~Haan, Davidson, and
  Forr{\'e}]{falorsi2019reparameterizing}
Luca Falorsi, Pim de~Haan, Tim~R Davidson, and Patrick Forr{\'e}.
\newblock Reparameterizing distributions on lie groups.
\newblock In \emph{The 22nd International Conference on Artificial Intelligence
  and Statistics}, pages 3244--3253. PMLR, 2019.

\bibitem[Ganea et~al.(2018)Ganea, B{\'e}cigneul, and
  Hofmann]{ganea2018hyperbolic}
Octavian Ganea, Gary B{\'e}cigneul, and Thomas Hofmann.
\newblock Hyperbolic neural networks.
\newblock \emph{Advances in neural information processing systems}, 31, 2018.

\bibitem[Guo et~al.(2019)Guo, Zhu, Liu, and Yin]{guo2019affine}
Xifeng Guo, En~Zhu, Xinwang Liu, and Jianping Yin.
\newblock Affine equivariant autoencoder.
\newblock In \emph{IJCAI}, pages 2413--2419, 2019.

\bibitem[Hall and Hall(2013)]{hall2013lie}
Brian~C Hall and Brian~C Hall.
\newblock \emph{Lie groups, Lie algebras, and representations}.
\newblock Springer, 2013.

\bibitem[Higgins et~al.(2017)Higgins, Matthey, Pal, Burgess, Glorot, Botvinick,
  Mohamed, and Lerchner]{higgins2017betavae}
Irina Higgins, Loic Matthey, Arka Pal, Christopher Burgess, Xavier Glorot,
  Matthew Botvinick, Shakir Mohamed, and Alexander Lerchner.
\newblock beta-{VAE}: Learning basic visual concepts with a constrained
  variational framework.
\newblock In \emph{International Conference on Learning Representations}, 2017.
\newblock URL \url{https://openreview.net/forum?id=Sy2fzU9gl}.

\bibitem[Higgins et~al.(2018)Higgins, Amos, Pfau, Racaniere, Matthey, Rezende,
  and Lerchner]{higgins_towards_2018}
Irina Higgins, David Amos, David Pfau, Sebastien Racaniere, Loic Matthey,
  Danilo Rezende, and Alexander Lerchner.
\newblock Towards a {Definition} of {Disentangled} {Representations}.
\newblock \emph{arXiv:1812.02230 [cs, stat]}, December 2018.
\newblock URL \url{http://arxiv.org/abs/1812.02230}.
\newblock arXiv: 1812.02230 version: 1.

\bibitem[Higgins et~al.(2022)Higgins, Racani{\`e}re, and
  Rezende]{higgins2022symmetry}
Irina Higgins, S{\'e}bastien Racani{\`e}re, and Danilo Rezende.
\newblock Symmetry-based representations for artificial and biological general
  intelligence.
\newblock \emph{Frontiers in Computational Neuroscience}, page~28, 2022.

\bibitem[Kingma and Welling(2014)]{kingma2014}
Diederik~P. Kingma and Max Welling.
\newblock {Auto-Encoding Variational Bayes}.
\newblock In \emph{2nd International Conference on Learning Representations,
  {ICLR} 2014, Banff, AB, Canada, April 14-16, 2014, Conference Track
  Proceedings}, 2014.

\bibitem[K{\"o}hler et~al.(2021)K{\"o}hler, Kr{\"a}mer, and
  No{\'e}]{kohler2021smooth}
Jonas K{\"o}hler, Andreas Kr{\"a}mer, and Frank No{\'e}.
\newblock Smooth normalizing flows.
\newblock \emph{Advances in Neural Information Processing Systems},
  34:\penalty0 2796--2809, 2021.

\bibitem[K{\"o}hler et~al.(2023)K{\"o}hler, Invernizzi, de~Haan, and
  No{\'e}]{kohler2023rigid}
Jonas K{\"o}hler, Michele Invernizzi, Pim de~Haan, and Frank No{\'e}.
\newblock Rigid body flows for sampling molecular crystal structures.
\newblock \emph{arXiv preprint arXiv:2301.11355}, 2023.

\bibitem[Kondor and Trivedi(2018)]{kondor2018generalization}
Risi Kondor and Shubhendu Trivedi.
\newblock On the generalization of equivariance and convolution in neural
  networks to the action of compact groups.
\newblock In \emph{International Conference on Machine Learning}, pages
  2747--2755. PMLR, 2018.

\bibitem[Lee et~al.(2021)Lee, Kwon, and Park]{lee2021neighborhood}
Yonghyeon Lee, Hyeokjun Kwon, and Frank Park.
\newblock Neighborhood reconstructing autoencoders.
\newblock \emph{Advances in Neural Information Processing Systems}, 34, 2021.

\bibitem[Lezcano-Casado and Mart{\i}nez-Rubio(2019)]{lezcano2019cheap}
Mario Lezcano-Casado and David Mart{\i}nez-Rubio.
\newblock Cheap orthogonal constraints in neural networks: A simple
  parametrization of the orthogonal and unitary group.
\newblock In \emph{International Conference on Machine Learning}, pages
  3794--3803. PMLR, 2019.

\bibitem[Liu et~al.(2019)Liu, Jiang, He, Chen, Liu, Gao, and
  Han]{liu2019variance}
Liyuan Liu, Haoming Jiang, Pengcheng He, Weizhu Chen, Xiaodong Liu, Jianfeng
  Gao, and Jiawei Han.
\newblock On the variance of the adaptive learning rate and beyond.
\newblock In \emph{International Conference on Learning Representations}, 2019.

\bibitem[Liu et~al.(2023)Liu, Liu, Yin, Wang, Chen, and Wang]{liu2023delving}
Yulin Liu, Haoran Liu, Yingda Yin, Yang Wang, Baoquan Chen, and He~Wang.
\newblock Delving into discrete normalizing flows on so (3) manifold for
  probabilistic rotation modeling.
\newblock \emph{arXiv preprint arXiv:2304.03937}, 2023.

\bibitem[Locatello et~al.(2020)Locatello, Tschannen, Bauer, R{\"a}tsch,
  Sch{\"o}lkopf, and Bachem]{locatello2020disentangling}
Francesco Locatello, Michael Tschannen, Stefan Bauer, Gunnar R{\"a}tsch,
  Bernhard Sch{\"o}lkopf, and Oliver Bachem.
\newblock Disentangling factors of variations using few labels.
\newblock In \emph{Eighth International Conference on Learning
  Representations}. OpenReview. net, 2020.

\bibitem[Marchetti et~al.(2022)Marchetti, Tegn{\'e}r, Varava, and
  Kragic]{marchetti2022equivariant}
Giovanni~Luca Marchetti, Gustaf Tegn{\'e}r, Anastasiia Varava, and Danica
  Kragic.
\newblock Equivariant representation learning via class-pose decomposition.
\newblock \emph{arXiv preprint arXiv:2207.03116}, 2022.

\bibitem[Mathieu and Nickel(2020)]{mathieu2020riemannian}
Emile Mathieu and Maximilian Nickel.
\newblock Riemannian continuous normalizing flows.
\newblock \emph{Advances in Neural Information Processing Systems},
  33:\penalty0 2503--2515, 2020.

\bibitem[Meng et~al.(2019)Meng, Huang, Wang, Zhang, Zhuang, Kaplan, and
  Han]{meng2019spherical}
Yu~Meng, Jiaxin Huang, Guangyuan Wang, Chao Zhang, Honglei Zhuang, Lance
  Kaplan, and Jiawei Han.
\newblock Spherical text embedding.
\newblock \emph{Advances in Neural Information Processing Systems}, 32, 2019.

\bibitem[Miao et~al.()Miao, Mathieu, Siddharth, Teh, and
  Rainforth]{miaoincorporating}
Ning Miao, Emile Mathieu, N~Siddharth, Yee~Whye Teh, and Tom Rainforth.
\newblock On incorporating inductive biases into vaes.
\newblock In \emph{International Conference on Learning Representations}.

\bibitem[Moor et~al.(2020)Moor, Horn, Rieck, and
  Borgwardt]{moor2020topological}
Michael Moor, Max Horn, Bastian Rieck, and Karsten Borgwardt.
\newblock Topological autoencoders.
\newblock In Hal~Daumé III and Aarti Singh, editors, \emph{Proceedings of the
  37th International Conference on Machine Learning}, volume 119 of
  \emph{Proceedings of Machine Learning Research}, pages 7045--7054. PMLR,
  13--18 Jul 2020.
\newblock URL \url{https://proceedings.mlr.press/v119/moor20a.html}.

\bibitem[Murphy et~al.(2021)Murphy, Esteves, Jampani, Ramalingam, and
  Makadia]{pmlr-v139-murphy21a}
Kieran~A Murphy, Carlos Esteves, Varun Jampani, Srikumar Ramalingam, and Ameesh
  Makadia.
\newblock Implicit-pdf: Non-parametric representation of probability
  distributions on the rotation manifold.
\newblock In Marina Meila and Tong Zhang, editors, \emph{Proceedings of the
  38th International Conference on Machine Learning}, volume 139 of
  \emph{Proceedings of Machine Learning Research}, pages 7882--7893. PMLR,
  18--24 Jul 2021.
\newblock URL \url{https://proceedings.mlr.press/v139/murphy21a.html}.

\bibitem[Park et~al.(2022)Park, Biza, Zhao, van~de Meent, and
  Walters]{park2022learning}
Jung~Yeon Park, Ondrej Biza, Linfeng Zhao, Jan~Willem van~de Meent, and Robin
  Walters.
\newblock Learning symmetric embeddings for equivariant world models.
\newblock In \emph{International Conference on Machine Learning}, 2022.

\bibitem[Perez~Rey et~al.(2020)Perez~Rey, Menkovski, and
  Portegies]{rey_diffusion_2019}
Luis~A. Perez~Rey, Vlado Menkovski, and Jim Portegies.
\newblock Diffusion {Variational} {Autoencoders}.
\newblock In Christian Bessiere, editor, \emph{Proceedings of the
  {Twenty}-{Ninth} {International} {Joint} {Conference} on {Artificial}
  {Intelligence}, {IJCAI}-20}, pages 2704--2710. International Joint
  Conferences on Artificial Intelligence Organization, July 2020.
\newblock \doi{10.24963/ijcai.2020/375}.
\newblock URL \url{https://doi.org/10.24963/ijcai.2020/375}.

\bibitem[Pfau et~al.(2020)Pfau, Higgins, Botev, and
  Racani{\`e}re]{pfau_disentangling_nodate}
David Pfau, Irina Higgins, Alex Botev, and S{\'e}bastien Racani{\`e}re.
\newblock Disentangling by subspace diffusion.
\newblock \emph{Advances in Neural Information Processing Systems},
  33:\penalty0 17403--17415, 2020.

\bibitem[Quessard et~al.(2020)Quessard, Barrett, and
  Clements]{quessard2020learning}
Robin Quessard, Thomas Barrett, and William Clements.
\newblock Learning disentangled representations and group structure of
  dynamical environments.
\newblock \emph{Advances in Neural Information Processing Systems}, 33, 2020.

\bibitem[Rezende and Mohamed(2015)]{rezende2015variational}
Danilo Rezende and Shakir Mohamed.
\newblock Variational inference with normalizing flows.
\newblock In \emph{International conference on machine learning}, pages
  1530--1538. PMLR, 2015.

\bibitem[Rezende et~al.(2014)Rezende, Mohamed, and
  Wierstra]{rezende2014stochastic}
Danilo~Jimenez Rezende, Shakir Mohamed, and Daan Wierstra.
\newblock Stochastic backpropagation and approximate inference in deep
  generative models.
\newblock In \emph{International conference on machine learning}, pages
  1278--1286. PMLR, 2014.

\bibitem[Rezende et~al.(2020)Rezende, Papamakarios, Racaniere, Albergo, Kanwar,
  Shanahan, and Cranmer]{rezende2020normalizing}
Danilo~Jimenez Rezende, George Papamakarios, S{\'e}bastien Racaniere, Michael
  Albergo, Gurtej Kanwar, Phiala Shanahan, and Kyle Cranmer.
\newblock Normalizing flows on tori and spheres.
\newblock In \emph{International Conference on Machine Learning}, pages
  8083--8092. PMLR, 2020.

\bibitem[Shilane et~al.(2004)Shilane, Min, Kazhdan, and
  Funkhouser]{shape_benchmark}
Philip Shilane, Patrick Min, Michael Kazhdan, and Thomas Funkhouser.
\newblock The princeton shape benchmark.
\newblock In \emph{Proceedings Shape Modeling Applications, 2004.}, 2004.

\bibitem[Tonnaer et~al.(2022)Tonnaer, Rey, Menkovski, Holenderski, and
  Portegies]{pmlr-v162-tonnaer22a}
Loek Tonnaer, Luis Armando~Perez Rey, Vlado Menkovski, Mike Holenderski, and
  Jim Portegies.
\newblock Quantifying and learning linear symmetry-based disentanglement.
\newblock In Kamalika Chaudhuri, Stefanie Jegelka, Le~Song, Csaba Szepesvari,
  Gang Niu, and Sivan Sabato, editors, \emph{Proceedings of the 39th
  International Conference on Machine Learning}, volume 162 of
  \emph{Proceedings of Machine Learning Research}, pages 21584--21608. PMLR,
  17--23 Jul 2022.
\newblock URL \url{https://proceedings.mlr.press/v162/tonnaer22a.html}.

\bibitem[Vadgama et~al.(2022)Vadgama, Tomczak, and Bekkers]{vadgama2022kendall}
Sharvaree Vadgama, Jakub~Mikolaj Tomczak, and Erik~J Bekkers.
\newblock Kendall shape-vae: Learning shapes in a generative framework.
\newblock In \emph{NeurIPS 2022 Workshop on Symmetry and Geometry in Neural
  Representations}, 2022.

\bibitem[Winter et~al.(2022)Winter, Bertolini, Le, Noe, and
  Clevert]{winter2022unsupervised}
Robin Winter, Marco Bertolini, Tuan Le, Frank Noe, and Djork-Arn{\'e} Clevert.
\newblock Unsupervised learning of group invariant and equivariant
  representations.
\newblock In Alice~H. Oh, Alekh Agarwal, Danielle Belgrave, and Kyunghyun Cho,
  editors, \emph{Advances in Neural Information Processing Systems}, 2022.
\newblock URL \url{https://openreview.net/forum?id=47lpv23LDPr}.

\bibitem[Xu and Durrett(2018)]{xu2018spherical}
Jiacheng Xu and Greg Durrett.
\newblock Spherical latent spaces for stable variational autoencoders.
\newblock In \emph{Proceedings of the 2018 Conference on Empirical Methods in
  Natural Language Processing}, pages 4503--4513, 2018.

\bibitem[Yang et~al.(2022)Yang, Ren, Wang, Zeng, and Zheng]{yang2022towards}
Tao Yang, Xuanchi Ren, Yuwang Wang, Wenjun Zeng, and Nanning Zheng.
\newblock Towards building a group-based unsupervised representation
  disentanglement framework.
\newblock In \emph{International Conference on Learning Representations}, 2022.
\newblock URL \url{https://openreview.net/forum?id=YgPqNctmyd}.

\bibitem[Zhu et~al.(2021)Zhu, Xu, and Tao]{zhu_commutative_2021}
Xinqi Zhu, Chang Xu, and Dacheng Tao.
\newblock Commutative lie group vae for disentanglement learning.
\newblock In \emph{International Conference on Machine Learning}, pages
  12924--12934. PMLR, 2021.

\end{thebibliography}
\bibliographystyle{plainnat}

% \section{Supplementary Material}
\appendix

\newpage
\section{Experimental Details}
\label{app:sec:exp}

% \paragraph{Experimental Details.} 
We train all our models for 150 epochs with a batch-size of 600. For optimization, we use the RAdam optimizer~\citep{liu2019variance} with a learning rate of 5e-4. 
% For low-dimensional cases, we use a 4-layer MLP with 512 hidden units followed by a $\tanh$ activation for both the encoder and decoder architecture (Table~\ref{app:tab:arch-crown}). 
For all image datasets, we use a 4-layer CNN with kernel, stride, and padding of size 4, 2 and 1 respectively followed by a $leakyReLU$ activation (Table~\ref{app:tab:arch-images}) for the encoder and a $Sigmoid$ activation for the decoder.
The network $\sigma_{\epsilon}^2$, shares the same architecture with the only difference being the last layer, which is a fully-connected network followed by a $Softplus$ activation which is common in standard VAEs.   
In all our experiments, we used $K=1$ for the GF-VAE models as it was sufficient to avoid the optimization obstructions mentioned in the paper.
All models were initialized and trained with 15\footnote{We used a higher number of random seeds than normal to account for the training instability.} different random seeds. 
% Though we hypothesize that these topological defects may occur in a variety of topological structures, in this paper we restrict our experiments to the Lie group $SO(2)$ with $\pi$ defined as in Section~\ref{sec:problem}.
% 

% \begin{table}[!h]
%     \centering
%     % \begin{subtable}[h]{.5\linewidth}
%     % \caption{MNIST and Fashion-MNIST.}
%     \centering
%         \begin{tabular}{|l|}
%         \toprule
%         \textbf{Encoder} \\
%         \midrule
%         Input $x \in \sR^3$  \\
%         \hline 
%         F.C. 512, Tanh.  \\
%         \hline 
%         F.C. 512, Tanh. \\
%         \hline 
%         F.C. 512, Tanh.  \\
%         \hline
%         F.C. 2, $\pi(y):= y/\|y\|$. \\
%         % \hline
%         %  \\
%         \bottomrule
%         \end{tabular}
%     % \end{subtable}%%%
%     % \begin{subtable}[h]{.5\textwidth}
%     % \caption{CIFAR10 and SVHN.}
%     \hspace*{3ex}
%     \centering
%         \begin{tabular}{|l|}
%         \toprule
%         \textbf{Decoder}  \\
%         \midrule
%         Input $z \in \sR^2$ s.t. $\|z\|_2 = 1$  \\
%         \hline 
%         F.C. 512, Tanh.  \\
%         \hline 
%         F.C. 512, Tanh. \\
%         \hline 
%         F.C. 512, Tanh.  \\
%         \hline 
%         F.C. 3.  \\ 
%         \bottomrule
%         \end{tabular}
%     % \end{subtable}
%     % \vspace*{-3ex}
%     \caption{Architecture of the Encoders and Decoders for the 3D crown dataset.}
%     \label{app:tab:arch-crown}
% \end{table}

\begin{table}[!h]
    \centering
    \caption{Architecture of the encoders and decoders employed for all image datasets.}
    % \begin{subtable}[h]{.5\linewidth}
    % \caption{MNIST and Fashion-MNIST.}
    % \centering
        \begin{tabular}{|l|}
        \toprule
        \textbf{Encoder} \\
        \midrule
        Input $32\times32$ images  \\
        \hline 
        $4\times4$ conv. 32 stride 2, LeakyReLU.  \\
        \hline 
        $4\times4$ conv. 32 stride 2, LeakyReLU. \\
        \hline 
        $4\times4$ conv. 64 stride 2, LeakyReLU.  \\
        \hline
        $4\times4$ conv. 64 stride 2, LeakyReLU. \\
        \hline
        F.C. 2, $\pi(y):= y/\|y\|$. \\
        \bottomrule
        \end{tabular}
    % \end{subtable}%%%
    % \begin{subtable}[h]{.5\textwidth}
    % \caption{CIFAR10 and SVHN.}
    \hspace*{3ex}
    % \centering
        \begin{tabular}{|l|}
        \toprule
        \textbf{Decoder}  \\
        \midrule
        Input $z \in \sR^2$ s.t. $\|z\|_2 = 1$  \\
        \hline 
        $4\times4$ deconv. 64,  stride 2, ELU.  \\
        \hline 
        $4\times4$ deconv. 64, stride 2, ELU. \\
        \hline 
        $4\times4$ deconv. 32, stride 2, ELU.  \\
        \hline
        $4\times4$ deconv. 3, \hspace{0.2em} stride 2, Sigmoid.  \\
        \bottomrule
        \end{tabular}
    % \end{subtable}
    % \vspace*{-3ex}
    
    \label{app:tab:arch-images}
\end{table}

% \newpage
% \section{Additional Traversal Figures}
% \label{app:sec:additional-traversals}

\section{Proofs}
\label{app:sec:proof}

We include the proofs for the propositions in the main text.

\subsection{Figure Eight Local Minimum}

\begin{proposition}  Assume that $f_{\phi(t)}$ undergoes continuous optimization and is thus continuous in $t$.  Assume that $\pi \circ f_{\phi(t)} \circ g$ is injective for all $t$.  The cyclic ordering induced on $k$ points by $f_{\phi(0)}$ is equal to that induced by $f_{\phi(1)}$. Thus a figure 8 embedding, which corresponds to cyclic order $(z_1,z_2,z_4,z_3)\ \mathrm{mod}\ C_4$, cannot be transformed to a homeomorphic embedding, which has cyclic order $(z_1,z_2,z_3,z_4)\ \mathrm{mod}\ C_4$ or $(z_4,z_3,z_2,z_1)\ \mathrm{mod}\ C_4$.
\end{proposition}

%https://scholar.harvard.edu/files/knudsen/files/lecture_1.pdf
\begin{proof}
Since we assume $\pi \circ f_{\phi(t)} \circ g$ is injective for all $t$, the path $\mathbf{z}(t) = (\pi \circ f_{\phi(0)} \circ g(\theta_i))_{i=1}^4$ is inside  
%The homotopy classes of $k$ distinct points $z_1,\ldots,z_k$ on $S^1$ are equal to the connected components of 
the $k$-fold configuration space on $S^1$ defined $\mathrm{Conf}_{k}(S^1) = \lbrace (z_1,\ldots,z_k) \in (S^1)^k : z_i \not = z_j \text{ for } i \not = j \rbrace$.  In order to prove the claim, we will show that the path-connected components of $\mathrm{Conf}_{k}(S^1)$ correspond to cyclic orderings of $(z_1,\ldots,z_k)$ and thus the start and end point of every path share a cyclic ordering.

Mapping $(z_1,\ldots,z_k) \mapsto (z_k,(z_k^{-1}z_1,\ldots,z_k^{-1}z_{k-1}))$ gives a homeomorphism $\mathrm{Conf}_k(S^1) \cong \mathrm{SO}(2) \times \mathrm{Conf}_{k-1}(S^1\setminus \lbrace 1 \rbrace) \cong \mathrm{SO}(2) \times \mathrm{Conf}_{k-1}(\mathbb{R})$.  Let $\tilde{z}_i = z_k^{-1} z_i$. Consider $D = \lbrace (\tilde z_1,\ldots,\tilde z_{k-1}) : \tilde z_1 < \ldots < \tilde z_{k-1} \rbrace \subset \mathrm{Conf}_{k-1}(\mathbb{R})$.  

We can identify the connected components of $\mathrm{Conf}_{k-1}(\mathbb{R})$. The set $D$ is a fundamental domain for the action of the symmetric group $S_{k-1}$ on $\mathrm{Conf}_{k-1}(\mathbb{R})$.  Thus $\mathrm{Conf}_{k-1}(\mathbb{R}) = \coprod_{\sigma \in S_k} \sigma(D)$ is a disjoint union.  Linear interpolation shows $D$ is connected.  The sets $D$ and $\sigma(D)$ for $\sigma \in S_k$ are not connected.  Consider a path from $\mathbf{z} = (z_1,\ldots,z_k) \in D$ to $\sigma(\mathbf{z}) \in \sigma(D)$. The element $\sigma$ must reverse the order of at least two elements $z_j < z_i$. Thus the function $f(\mathbf{z}) = z_i - z_j$ must take the value 0 over the path by intermediate value theorem.  Hence the path cannot be in $\mathrm{Conf}_{k-1}(\mathbb{R})$.  Thus the connected components of  $\pi_0(\mathrm{Conf}_{k-1}(\mathbb{R})) \cong S_{k-1}$.  

Since $\mathrm{SO}(2)$ is connected, $\pi_0(\mathrm{Conf}_{k}(S^1)) \cong S_{k-1}$. 
That is each connected component of $\mathrm{Conf}_{k}(S^1)$ is labeled by an element of $S_{k-1}$ describing the ordering of $\tilde{z}_1,\ldots,\tilde{z}_{k-1}$ in $\mathbb{R}$.  Each ordering of $(\tilde{z}_1,\ldots,\tilde{z}_{k-1})$ in turn corresponds to a different cyclic ordering of $z_1,\ldots,z_k$ in $S^1$, that is, a different element of $S_k / C_k$.  Thus two $k$-point configurations are homotopic if and only if they have the same cyclic ordering.  
\end{proof}

\subsection{Degree Obstruction}

\begin{proposition}
The winding number of the initialized model and final model are equal $w(\pi' \circ h_{\phi(1)} \circ g) =  w(\pi' \circ h_{\phi(1)} \circ g)$.
\end{proposition}

\begin{proof}
The winding number of a map is a continuous function $t \mapsto w(\pi' \circ h_{\phi(t)} \circ g)$.  Since the output space $\mathbb{Z}$ is discreet, the winding number must be constant in $t$. 
\end{proof}

\subsection{Magnitude Growth in $\gY$}
% As noted in the main text, we have $\tilde{\nabla}_y \gL = (dF^T dF) \nabla_y\gL$. 
\emph{We assume that $dF$ is full rank}, which is a reasonable assumption for an overparameterized neural network.
%That is, we are assuming a locally surjective mapping from the model parameters to the intermediate outputs space $\gY = \mathbb{R}^2$.   
In that case $M = dF^T dF$ is a positive definite symmetric matrix and can be orthogonally diagonalized $M = Q \Lambda Q^T$ where $Q$ is orthogonal and  
\[
\Lambda = \begin{pmatrix}
\lambda_1 & 0 \\
0 & \lambda_2 
\end{pmatrix}
\]
and $\lambda_i > 0$.
The maximum angle between $\bm{x} = \nabla_y \gL$ and $M\bm{x} = \tilde{\nabla}_y \gL $ can then computed in terms of the eigenvalues $\lambda_i$. This maximum is computed for the case of an $n\times n$ symmetric positive definite matrix here\footnote{karakusc (https://math.stackexchange.com/users/176950/karakusc), Maximum angle between a vector $x$ and its linear transformation $A x$, URL (version: 2017-05-06): https://math.stackexchange.com/q/2266057}.  We include the proof for the $2 \times 2$ case we consider here for completeness.

\begin{lemma}
The maximum angle between ${x}$ and $M {x}$ for ${x} \in \mathbb{R}^2_{\not = 0}$ is 
\[
\cos^{-1} \left( \frac{2 \sqrt{\lambda_1 \lambda_2}}{\lambda_1 + \lambda_2}\right).
\]
\end{lemma}

\begin{proof}
The angle is maximized at the minimum value of \[
\frac{x^T M x}
{\| x \| \| M x\|}.
\]
It suffices to consider $\|x \| = 1$.  Substituting $M = Q \Lambda Q^T$ and $y = Q x$, we want to minimize
\[
\frac{x^T Q^T \Lambda Q x}
{ x^T Q^T \Lambda^2 Q x} = \frac{y^T \Lambda y}
{ y^T \Lambda^2 y} 
\]
over all $\| y \| = 1$ since $\| Q x \| = \|x\| = 1$.
Letting $a = y_1^2$ and noting $y_1^2+y_2^2 = 1$, this is equal to minimizing 
\[
\frac{a \lambda_1 + (1-a) \lambda_2}{a \lambda_1^2 + (1-a) \lambda_2^2}
\]
over $0 \leq a \leq 1$.  Setting the derivative equal to 0 gives
\[
\frac{(\lambda_1-\lambda_2)^2 (-\lambda_2 + a(\lambda_1+\lambda_2))}{2 \left(\lambda_2^2 + a (\lambda_1^2 - \lambda_2^2))\right)^{3/2}} = 0
\]
and yields one critical value at $a = \lambda_2 / (\lambda_1 + \lambda_2)$ corresponding to value $ \frac{2 \sqrt{\lambda_1 \lambda_2}}{\lambda_1 + \lambda_2}$.  This is the global minimum since the boundary values $a = 0$ and $a=1$ correspond to maxima with value 1.
\end{proof}

Thus the angle between $\nabla_y \gL$ and $\tilde{\nabla}_y \gL$ is bounded by $\theta =  \cos^{-1} \left( 2 \sqrt{\lambda_1 \lambda_2}/(\lambda_1 + \lambda_2)\right).$
  Given that $\nabla_y \gL$ is tangential to the circle, assuming for simplicity $\tilde{\nabla}_y \gL$ has constant length $L$ and uniform distribution $[-\theta,\theta]$ in angle to $\nabla_y \gL$, the norm of $y$ grows \emph{in expectation}.

\begin{proof}
Without loss of generality, $y=(0,R)$ and $v = (L \cos{t}, L \sin{t} )$ where $|t| < \theta$.  Then we evaluate
\begin{align*}    
\mathbb{E}[\|y + v\|^2] &= \frac{1}{2\theta} \int_{-\theta}^\theta \|(L \cos{t}, R + L \sin{t}) \|^2 \| dt \\
&= \frac{1}{2\theta} \int_{-\theta}^\theta (L^2 \cos^2{t} + R^2 + 2 R L \sin{t} + L^2 \sin^2{t}) dt  \\
&= \frac{1}{2\theta} \int_{-\theta}^\theta (L^2 + R^2) dt \\
&= L^2+R^2.
\end{align*}
by Pythagorean identity and the fact $\sin{t}$ is odd. 
\end{proof}

% \begin{proposition}
% Assume a circle of radius $R$.  Let $v$ be a random vector at $y$ on the circle of length $L < R$ with angle to the tangent uniform in $[-\theta,\theta]$.  Then
% \[
% \mathbb{E}[\|y + v\|^2] = \frac{\phi  \left(L^2+2 R^2\right)+L \sin (\phi ) (2-L \cos (\phi ))}{2 \phi }  > \| y \|^2.
% \]
% \end{proposition}

% \section{Comparison to LSBD objective}
% \label{sec:lsbd_objective}

\section{Continuity Metric}
\label{app:sec:cont}

For measuring continuity, we adopt a similar method as~\citeauthor{falorsi_explorations_2018} and evaluate continuity in terms of how the largest ``jump'' compare to others when walking a continuous path $m_i \in \gM$ for $i = 1 \cdots N$ pairwise close points. We define $q_i$ to be the ratio distances between the ground-truth and the learned representation
\begin{equation*}
    q_i = \frac{d_{\gM}(\eta_{\phi}(m_i), \eta_\phi(m_{i+1}))}{d_{\gM}(m_{i}, m_{i+1})}.
\end{equation*}
From the set $\{q_i\}_i$, we compute the continuity metric $L_{\text{cont}}$ as
\begin{equation}
    L_{\text{cont}} = \frac{M}{P_\alpha}, \qquad M = \max_i q_i, \qquad P_{\alpha} = \alpha\text{-th percentile of }\{q_i\}_{i=1}^{N}.
\end{equation}
In our experiments, we set $\alpha=90$.

There are two differences between how we evaluate continuity compared to~\citep{falorsi_explorations_2018}. First, we measure the continuity of $\psi_\phi$ rather than $f_\phi$, which we argue is more relevant. Second, in~\citep{falorsi_explorations_2018}, the authors are mainly interested in verifying whether the encoder is discontinuous in the topological sense (which they verify by examining the inequality $M > \gamma P_{\alpha}$ for some $\gamma$). We on the other hand report continuity on the spectrum by computing the $\gamma$ that would make $\eta_\phi$ discontinuous. For evaluating homeomorphism, we conclude for an encoder to be homeomorphic if $L_{\text{cont}} < 10$  (empirically, we observed the mapping to appear smooth for a continuity score below this threshold).

\newpage 

\section{Decoding From $\gY$}
\label{app:decodefromy}

We consider the alternate strategy of removing the projection $\pi$ and adding a loss so that $y$ stays close to the desired manifold $\gM$ in $\gY$.  Winding number obstructions become far more prominent in this case. In Figure~\ref{app:fig:teapot-so2-y-decode-latents}, we show the latent space for different random seeds in the teapot case when we train with such an objective. 

\begin{figure}[!h]
\centering
\includegraphics[width=0.95\textwidth]{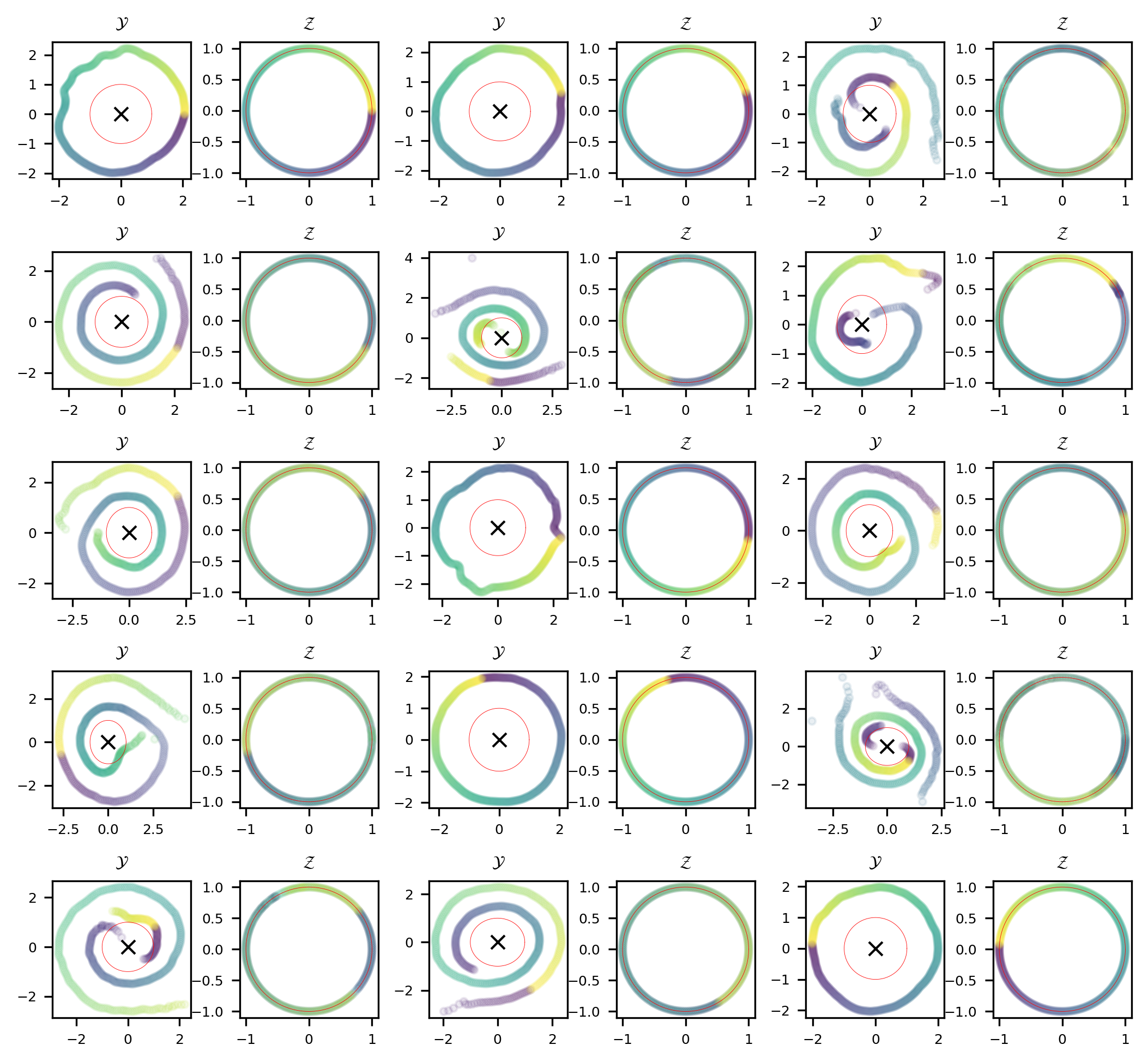}
\caption{$\gY$ and $\gZ$ space of Isom-AEs trained on teapots for 15 random seeds, where instead of decoding from $\gZ$, we decode from $\gY$ with an additional soft regularization that constrain $y$ values to have unit length.}
\label{app:fig:teapot-so2-y-decode-latents}
\end{figure}

\newpage

\section{Experiments: Additional Results}
\label{app:sec:additional-results}

We report the full results of our $\mathrm{SO}(2)$ experiments in Table~\ref{app:tab:image-metrics}.
% We observe that overall, GF-VAE is able to converge to a homeomorphic mapping more reliably compared to VAEs. We also observe that increasing the $\beta$ value for the KL term generally helps. This is perhaps unsurprising given that a high $\beta$ value encourages the latent space to cover the prior and therefore discourages the winding numbers from being 0 and improves the performance on the continuity metric. In the case of the teapots, we in fact observe that increasing $\beta$ is sufficient to learn a homeomorphic encoder. However, GF-VAE still scores better in terms of continuity as GF-VAE obstructions is less of a figure eight type compared to the obstructions for VAEs. 

\begin{table}[!h]
\centering
\caption{Comparison of different VAEs trained on the various image datasets in terms of number of encoders with homeomorphic mappings (\# H.), correct winding number (\# W.), and correct crossing number (\# C.) for 15 random seeds. We additionally report the error on continuity as well as negative Loglilkelihood (lower is better).}
\footnotesize
\begin{tabular}{llllll}
\multicolumn{6}{c}{\emph{L-shaped Tetrominoes}} \\ 
                    & \# H. & \# W.    & \# C. & Continuity    & $-$Loglilkelihood   \\
\midrule
AE		                     & 2/15		 & 2/15		 & 2/15	 & $137.28		\pm 54.15$ 	 & 	 $9.66\pm 5.63$		 \\
VAE $(\beta = 1)$		     & 0/15		 & 7/15		 & 0/15	 & $117.02		\pm 23.17$ 	 & 	 $7.26\pm 6.99$		 \\
% Action-VAE $(\beta = 1)$	 & 4/15		 & 8/15		 & 4/15	 & $88.86		\pm 55.78$ 	 & 	 $1.05\pm 0.52$		 \\
VAE $(\beta = 4)$		     & 2/15		 & 9/15		 & 2/15	 & $132.45		\pm 58.79$ 	 & 	 $\boldsymbol{2.67\pm 0.30}$		 \\
% Action-VAE $(\beta = 4)$	 & 2/15		 & 10/15	 & 2/15	 & $116.56		\pm 48.08$ 	 & 	 $2.86\pm 0.39$		 \\
GF-VAE $(\beta = 1)$		 & 5/15		 & 14/15	 & --	 & $50.51		\pm 58.18$ 	 & 	 $2.71\pm 1.67$		 \\
% Action-GF-VAE $(\beta = 1)$& 4/15		 & 8/15		 & --	 & $154.69		\pm 147.88$ 	 & 	 $6.84\pm 6.52$		 \\
GF-VAE $(\beta = 4)$		 & 9/15		 & \textbf{15/15}	 & --	 & $24.04		\pm 29.14$ 	 & 	 $4.43\pm 1.84$		 \\
Action-GF-VAE $(\beta = 4)$	 & \textbf{10/15}	 & 14/15	 & --	 & $\boldsymbol{19.35		\pm 23.37}$ 	 & 	 $4.43\pm 1.67$		 \\
\midrule
\midrule
\multicolumn{6}{c}{\emph{Teapots}} \\ 
                    & \# H. & \# W.    & \# C. & Continuity     & $-$Loglilkelihood.   \\
\midrule
AE		                         & 0/15		 & \textbf{15/15}		 & 8/15	 & $173.37		\pm 21.99$ 	 & 	 $\boldsymbol{7.49\pm 0.16}$		 \\
VAE $(\beta = 1)$		         & 6/15		 & \textbf{15/15}		 & 8/15	 & $14.21		\pm 7.30$ 	 & 	 $7.61\pm 0.07$		 \\
% Action-VAE $(\beta = 1)$	     & 15/15		 & 15/15		 & 0/15	 & $3.95		\pm 0.58$ 	 & 	 $5.77\pm 0.04$		 \\
VAE $(\beta = 4)$		         & \textbf{15/15}		 & \textbf{15/15}		 & \textbf{15/15}	 & $\boldsymbol{1.22		\pm 0.05}$ 	 & 	 $10.60\pm 0.01$		 \\
% Action-VAE $(\beta = 4)$	     & 15/15		 & 15/15		 & 15/15	 & $1.22		\pm 0.06$ 	 & 	 $10.49\pm 0.00$		 \\
GF-VAE $(\beta = 1)$		     & 7/15		 & \textbf{15/15}		 & --	 & $21.07		\pm 24.42$ 	 & 	 $8.48\pm 0.38$		 \\
% Action-GF-VAE $(\beta = 1)$	 & 5/15		 & 15/15		 & 0/15	 & $41.62		\pm 35.79$ 	 & 	 $7.65\pm 0.85$		 \\
GF-VAE $(\beta = 4)$		     & 13/15		 & 14/15		 & --	 & $8.16		\pm 17.89$ 	 & 	 $11.05\pm 0.24$		 \\
Action-GF-VAE $(\beta = 4)$		 & 13/15		 & 12/15		 & --	 & $3.14		\pm 1.74$ 	 & 	 $10.96\pm 0.49$		 \\
\midrule
\midrule
\multicolumn{6}{c}{\emph{Airplanes}} \\ 
                    & \# H. & \# W.    & \# C. & Continuity    & $-$Loglilkelihood   \\
\midrule
AE		                           & 0/15		 & 3/15		 & 3/15	 & $132.95		\pm 40.61$ 	 & 	 $11.55\pm 2.19$		 \\
VAE $(\beta=1)$		               & 0/15		 & 6/15		 & 4/15	 & $86.72		\pm 65.12$ 	 & 	 $12.61\pm 3.41$		 \\
% Action-VAE $(\beta = 1)$		   & 12/15		 & 12/15		 & 1/15	 & $21.83		\pm 37.21$ 	 & 	 $8.11\pm 1.51$		 \\
VAE $(\beta = 4)$		           & 5/15		 & 9/15		 & \textbf{5/15}	 & $67.88		\pm 59.20$ 	 & 	 $12.53\pm 1.77$		 \\
% Action-VAE $(\beta = 4)$		   & 13/15		 & 14/15		 & 13/15	 & $21.55		\pm 51.70$ 	 & 	 $10.70\pm 0.02$		 \\
GF-VAE $(\beta = 1)$		       & 0/15		 & 11/15		 & --	 & $63.94		\pm 34.43$ 	 & 	 $\boldsymbol{10.88\pm 1.24}$		 \\
% Action-GF-VAE $(\beta = 1)$	   & 0/15		 & 7/15		 & 0/15	 & $92.45		\pm 56.53$ 	 & 	 $11.74\pm 2.24$		 \\
GF-VAE $(\beta = 4)$		       & 7/15		 & \textbf{13/15}		 & --	 & $\boldsymbol{27.70		\pm 29.23}$ 	 & 	 $13.17\pm 1.64$		 \\
Action-GF-VAE $(\beta = 4)$		   & \textbf{9/15}		 & 12/15		 & --	 & $33.26		\pm 40.03$ 	 & 	 $12.27\pm 1.79$		 \\
\midrule
\end{tabular}
\vspace{0.25\baselineskip}
\vspace{0.25\baselineskip}
\label{app:tab:image-metrics}
\end{table}

\section{Additional Figures}

\begin{figure}[!h]
\centering
\vspace*{-1.0ex}
\includegraphics[width=\textwidth]{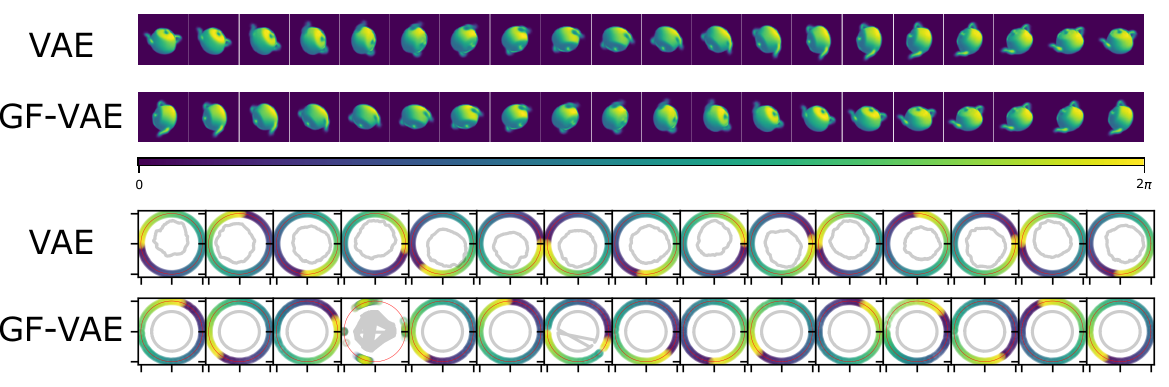}
\vspace{1.25ex}
\includegraphics[width=\textwidth]{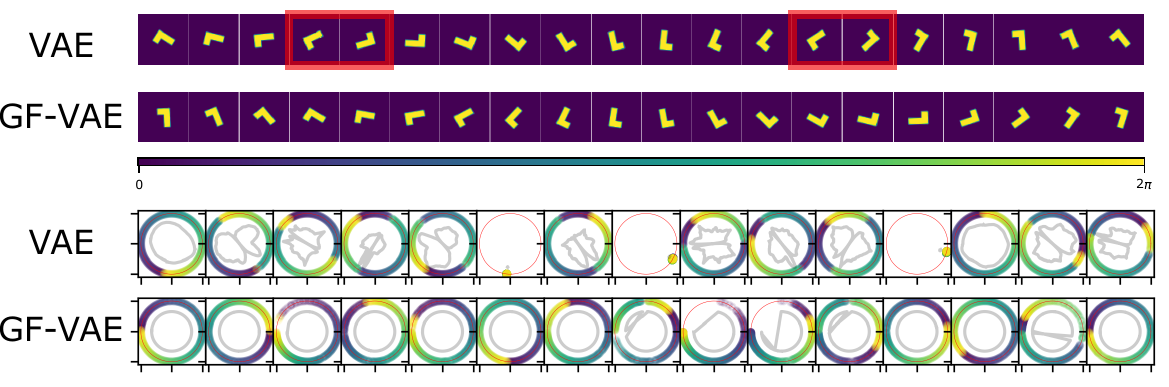}
\vspace*{-1.45ex}
\caption{\emph{Top}: Latent traversals in the decoder for a VAE and GF-VAE after training with $\beta=4$. \emph{Bottom:} The representations in latent space $\gZ$ for VAEs and GF-VAEs initialized with 15 different random seeds. For VAEs, the gray line inside the circle show the (scaled-down) $y$-traversals. For GF-VAEs, we it is difficult to visualize $\gY$ space as it is high-dimensional space. Therefore, to inspect for obstructions, we show the traversal of the $z$ vectors instead. }
% \caption{{\color{red} TODO: PLACE HOLDER; UPDATE FIGURE FOR GF-VAE RESULT. Latent space interpolation in the decoder (\emph{Top}) and encoder (\emph{Bottom}).}}
\vspace*{-1.25ex}
\label{app:fig:interpolation}
\end{figure}

\begin{figure}
    \centering
    \includegraphics[width=\textwidth]{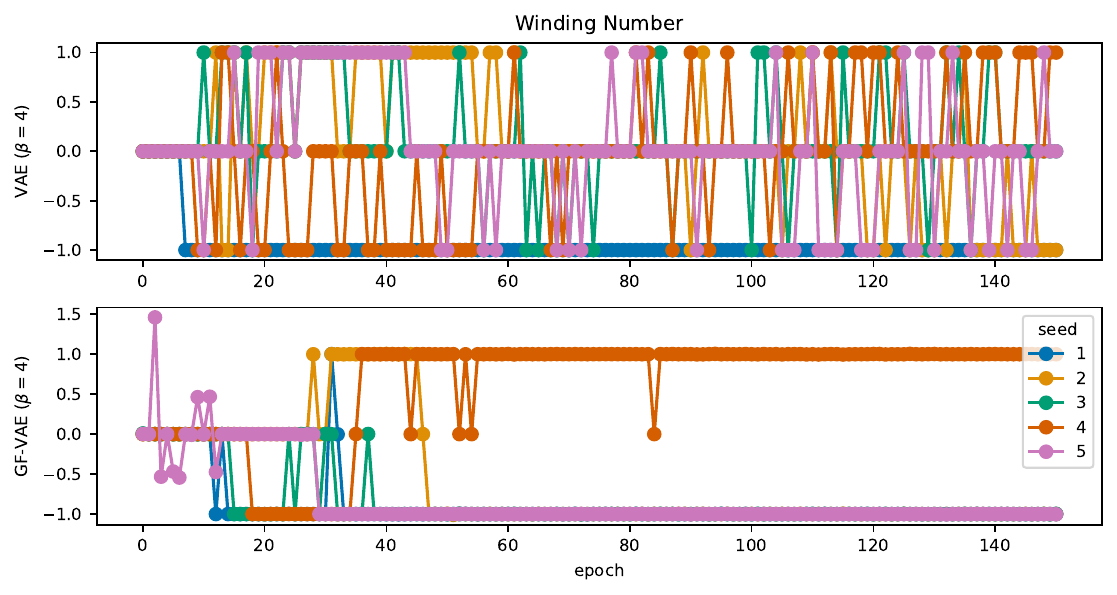}
    % \vspace*{-2.25ex}
    \caption{The winding numbers for VAE and GF-VAE trained with 5 different random seeds on the tetromino dataset as a function of epoch. }
    % \vspace*{-1.25ex}
    \label{app:fig:winding-numbers}
\end{figure}

\begin{figure}[!h]
\centering
\includegraphics[width=\textwidth]{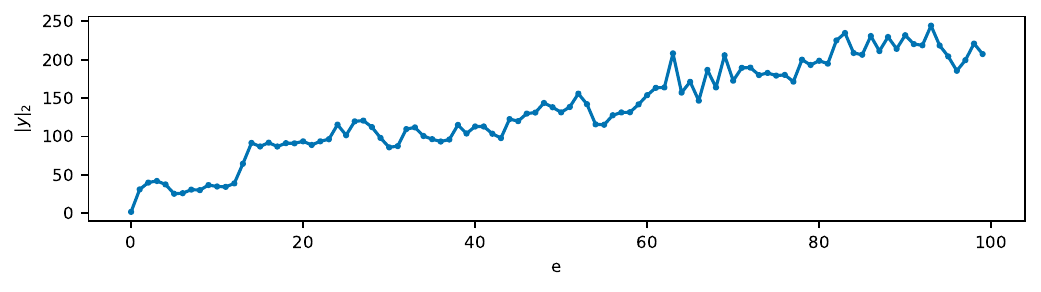}
\caption{$\|y\|_2$ as a function of epoch for a standard autoencoder trained on teapots for seed 0.}
\label{app:fig:ys-grwoth}
\end{figure}

\end{document}